\documentclass{article}

\usepackage{amsthm,amsmath}
\usepackage{xcolor}
\usepackage{xspace}
\usepackage[algo2e,ruled,vlined,linesnumbered]{algorithm2e}
\usepackage{tikz}

\newtheorem{theorem}{Theorem}

\newtheorem{corollary}[theorem]{Corollary}
	
\newtheorem{gp-problem}[theorem]{Problem}

\newcommand{\RLSGP}{RLS-GP\xspace}
\newcommand{\RLSGPStar}{RLS-GP$^*$\xspace}
\newcommand{\OneOneGP}{${(1+1)}$~$\text{GP}$\xspace}
\newcommand{\OneOneGPStar}{${(1+1)}$~$\text{GP}^*$\xspace}
\newcommand{\GPCBoolOperatorOR}{{\small OR}}
\newcommand{\GPCBoolOperatorAND}{{\small AND}}
\newcommand{\Order}{{\sc Order}\xspace}
\newcommand{\WOrder}{{\sc WOrder}\xspace}
\newcommand{\WMajority}{{\sc WMajority}\xspace}
\newcommand{\Majority}{{\sc Majority}\xspace}
\newcommand{\Sorting}{{\sc Sorting}\xspace}
\newcommand{\IdentificationP}{{\sc Identification}\xspace}
\newcommand{\onemax}{\textsc{OneMax}}

\definecolor{gpcgreen}{HTML}{29CC00}
\definecolor{gpcblue}{HTML}{ADE1F4}
\definecolor{gpcorange}{HTML}{FF6432}

\title{Computational Complexity Analysis of Genetic Programming}
\author{Andrei Lissovoi\footnote{Rigorous Research, Department of Computer Science, University of Sheffield, 211 Portobello, Sheffield~S1~4DP, UK.} \and Pietro S. Oliveto$^*$}
\begin{document}
\maketitle

\abstract{
Genetic programming (GP) is an evolutionary computation technique to
solve problems in an automated, domain-independent way. 
Rather than identifying the optimum of a function as in more traditional evolutionary optimization,
the aim of GP is to evolve computer programs with a given functionality.
While many GP applications have produced human competitive results, the theoretical understanding of what problem characteristics and algorithm properties 
allow GP to be effective is comparatively limited.
Compared with traditional evolutionary algorithms for function optimization, GP applications are further complicated by two additional factors: 
the variable-length representation of candidate programs, and the difficulty
of evaluating their quality efficiently. Such difficulties considerably impact the runtime analysis of GP, where space complexity also comes into play.
As a result, initial complexity analyses of GP have focused on restricted settings such as the evolution of trees with given structures or the estimation of solution quality using 
only a small polynomial number of input/output examples. 
However, the first computational complexity analyses of GP for evolving proper functions with defined input/output behavior have recently appeared.
In this chapter, we present an overview of the state of the art.
}

\numberwithin{equation}{section}
	
\section{Introduction}

Genetic programming (GP) is a class of evolutionary
computation techniques to evolve computer programs popularized by
Koza~\cite{KozaBook92}. GP uses genetic algorithm mutation, crossover and selection operators adapted to work on populations of program structures.
Program fitness is evaluated using
a \emph{training set} consisting of samples of program inputs and the
corresponding correct outputs. The goal of a GP system is to construct a
program which, as
well as producing the correct outputs on the inputs included in the training
set, generalizes well to other possible inputs.

In standard tree-based GP, as popularized by Koza, programs are expressed as syntax trees rather than
lines of code, with variables and constants (collectively referred to as
\emph{terminals}) appearing as leaf nodes in the tree, and functions (such as +, *,
and \texttt{cos}) appearing as internal nodes. New programs are produced by
mutation (which makes some changes to a solution) or crossover
(which creates new solutions by combining subtrees of two parent solutions). Several other variants of GP exist that use different representations than tree structures.
Popular ones are Linear GP~\cite{BrameierBanzhaf07book}, Cartesian GP~\cite{Miller11book}, and Geometric Semantic GP (GSGP)~\cite{MoraglioKJ12}.
Since most of the available computational complexity analyses focus on tree-based GP, this is where we keep our focus in this chapter.
Work on GSGP is an exception that we will also consider~\cite{MoraglioMM13}.

One of the main points regarding GP made by Koza is that a wide
variety of different problems from many different fields can be recast as
requiring the discovery of a computer program that produces some desired output
when presented with particular inputs \cite{KozaBook92}. Ideally, this process of discovery could
take place without requiring a human to explicitly make decisions about the
size, shape, or structural complexity of the solutions in advance. Since GP
systems provide a way to search the space of computer programs for
one which solves (or approximates) the problem at hand,
they are thus applicable to a wide variety of problems, including those in
artificial intelligence, machine learning, adaptive systems, and automated
learning. GP has produced human-competitive results and patentable solutions on a
large number of diverse problems, including the design of quantum computing
circuits \cite{Spector04:book}, antennas \cite{LohnHL08}, mechanical systems \cite{Lipson08}, and optical lens systems \cite{KozaAJ08}. From these results,
Koza observes that GP may be especially productive in areas where little
information about the size or shape of the ultimate solution is known, while large amounts of data and good
simulators are available to measure the performance of candidate solutions \cite{Koza10}.

While there are many examples of successful applications of GP (see \cite{Koza10} for an overview), the understanding of how such systems work and on which
problems they are successful is much more limited. 
Compared with traditional evolutionary algorithms for function optimization, GP applications are further
complicated by two additional factors: the variable-length representation of
candidate programs, and the difficulty of evaluating their quality efficiently, since it is prohibitive or even impossible to test programs on all possible inputs.
Such difficulties, naturally, impact the runtime analysis of GP considerably, where space
complexity also comes into play. As a result, while nowadays the analysis of standard elitist~\cite{CorusO17,CorusOY19} and 
nonelitist genetic algorithms~\cite{OlivetoW14,OlivetoW15,CorusDEL17} has finally become a reality,
analyzing standard GP systems is far more prohibitive. Indeed, McDermott and O'Reilly~\cite{McDermottO15} remarked that ``due to stochasticity, it is arguably impossible in most cases to
make formal guarantees about the number of fitness evaluations needed for a GP
algorithm to find an optimal solution.'' Similarly to how the analysis of simplified evolutionary algorithms (EAs) has gradually led to the achievement of techniques that nowadays
allow the analysis of standard EAs, Poli~et~{al.} suggested that ``computational complexity techniques being used to model simpler
GP systems, perhaps GP systems based on mutation and stochastic hill-climbing''
\cite{PoliVLM10}.

Following this guideline the first runtime analyses laying the groundwork for better understanding of GP considered
simplified algorithms primarily based on mutation and hill-climbing (i.e., the \OneOneGP algorithm introduced in \cite{DurrettNO11}). 
However, further simplifications compared with applications of GP in practice were necessary
to deal with the additional difficulties introduced by the variable length
of GP solutions, the stochastic fitness function evaluations when dynamic training sets were used, and the neighborhood structure imposed by
the GP mutation and crossover operations acting on syntax trees.
Indeed, Goldberg~and~O'Reilly observed that
``the methodology of using deliberately designed
problems, isolating specific properties, and pursuing, in detail, their
relationships in simple GP is more than sound; it is the only practical means
of systematically extending GP understanding and design'' \cite{GoldbergO98}.
To this end, the first runtime analyses of GP considered the time required to evolve particular tree
structures rather than proper computer programs. In particular, solution fitness was evaluated based on the tree structure rather than by executing the
evolved syntax tree. Problems belonging to this category are \Order, \Majority~\cite{DurrettNO11} and \Sorting~\cite{WagnerN12}.
Even in such simplified settings, the characteristic GP problem, bloat (i.e.,
the continuous growth of evolved solutions that is not accompanied by 
significant improvements in solution quality), may appear.

In GP applications generally, either the set of all possible inputs is too large
to evaluate the exact solution quality efficiently, or not much of it is known (i.e., only a limited amount
of information about the correct input/output behavior is available). 
As a result, the performance of the GP system is usually considered in the
probably approximately correct (PAC) learning framework \cite{Valiant84}, to show that the solution produced by the GP system generalizes well to
all inputs. {K\"otzing~et al.} isolated this issue when they presented the first runtime analysis of
a GP system in this framework~\cite{KotzingNS11}. They considered the problem of learning the weights assigned to $n$ bits of a
pseudo-Boolean function (i.e., the \IdentificationP problem), and proved that a simple GP system can discover the weights efficiently
even if a limited sample of the possible inputs is used to evaluate solution
quality.

A more realistic problem  where the program output,
rather than structure, is used as the basis for determining solution quality is the MAX problem~\cite{KotzingSNO14}, originally introduced in \cite{GathercoleR96}.
The problem is to evolve a program which, given some mathematical operators and
constants (the problem admits no variable inputs), outputs the maximum possible
value subject to a constraint on program size.

Only recently, the time and space complexity of the \OneOneGP has been analyzed for evolving Boolean functions of arity
$n$~\cite{MambriniO16,LissovoiO18}. Solution quality was evaluated by
comparing the output of the evolved programs with the target function on all
possible inputs, or on a polynomially sized training set. The analyses show that while
conjunctions of $n$ variables can be evolved efficiently (either exactly, using
the complete truth table as the training set, or in the PAC learning framework
when smaller training sets are used), parity functions of $n$ variables cannot.
These results represent the first rigorous complexity analysis of a tree-based GP system
for evolving functions with actual input/output behavior.

We will also consider the theoretical work on GSGP, where the variation 
operators used by the GP system are designed to modify program semantics
rather than program syntax.

This chapter presents an overview of the state of the art.
It is structured as follows. In
Section~\ref{sec:gp-prelim}, we introduce the \OneOneGP, the GP system used for
most of the available computational complexity analysis results. In
Section~\ref{sec:gp-tree-structure}, we present an overview of
the analyses of GP systems
for evolving tree structures with specific properties (the \Order,
\Majority, and \Sorting problems).
In Section~\ref{sec:gp-towards-programs}, we present results where GP
systems evolve programs with limited functionality: the MAX problem
is considered in Subsection~\ref{sec:gp-max},
and the \IdentificationP problem in Subsection~\ref{sec:GP-PAC}.
Section~\ref{sec:gp-boolean} presents results for GP evolving proper Boolean functions of arity $n$.
Section~\ref{sec:gp-other-algs} presents a brief overview of
the computational complexity results available for GSGP
algorithms. Finally, Section~\ref{sec:gp-conclusion}
presents a summary of the presented results and discusses the open directions
for future work.

\section{Preliminaries}\label{sec:gp-prelim}

In this chapter, we will primarily consider the behavior of the simple (1+1)~GP
algorithm (Algorithm~\ref{alg:OneOneGP}), which represents programs using
syntax trees and uses the HVL-Prime operator (Algorithm~\ref{alg:HVLPrime}) to perform mutations. This
algorithm maintains a population of one individual (initialized either with an
empty tree, or with a randomly generated tree),
and at each generation chooses between the parent
and a single offspring generated by HVL-Prime mutation. This simple
algorithm had already been considered in early comparative work between standard
tree-based GP and iterated hill-climbing versions of GP
\cite{OReillyPhD,OReillyO94,OReillyO96}.

\begin{algorithm2e}[t]
	Initialize a tree $X$\;
	\For{$t \gets 1, 2, \ldots$}{
		$X' \gets X$\;
		$k \gets 1+\mathrm{Poisson}(1)$\;
		\For{$i \gets 1, \ldots, k$}{
			$X' \gets \text{HVL\text-Prime}(X')$\;
		}
		\If{$f(X') \leq f(X)$}{
			$X \gets X'$\;
		}
	}
	\caption{The (1+1) GP} \label{alg:OneOneGP}
\end{algorithm2e}

The HVL-Prime mutation operator, introduced in \cite{DurrettNO11} and shown in
Algorithm~\ref{alg:HVLPrime} here, is an updated version of the HVL
(hierarchical variable length) mutation operator \cite{OReillyO94}. It is specialized to
deal with binary trees and is designed to perform similarly to bitwise
mutation in evolutionary algorithms. The original motivation for using the
HVL-Prime operator was that of making the smallest alterations possible to GP
trees while respecting the key properties of the GP tree search space: variable length and hierarchical structure.

\begin{algorithm2e}[t]
	\KwData{A binary syntax tree $X$.}
	Choose $op \in \{\text{INS}, \text{DEL}, \text{SUB}\}$ uniformly at random\;
	\uIf{$X$ is an empty tree}{
		Choose a literal $l \in L$ uniformly at random\;
		Set $l$ to be the root of $X$\;
	} \uElseIf{$op = \text{INS}$}{
		Choose a node $x \in X$ uniformly at random\;
		Choose $f \in F, l \in L$ uniformly at random\;
		Replace $x$ in $X$ with $f$\;
		Set the children of $f$ to be $x$ and $l$, order chosen uniformly at random\;
	} \uElseIf{$op = \text{DEL}$}{
		Choose a leaf node $x \in X$ uniformly at random\;
		Replace $x$'s parent in $X$ with $x$'s sibling in $X$\;
	} \ElseIf{$op = \text{SUB}$}{
		Choose a node $x \in X$ uniformly at random\;
		Choose a replacement $l \in L$, or $f \in F$ uniformly at random\;
		Replace $x$ in $X$ with $l$ if $x$ is a leaf node, or with $f$ if $x$ is an internal node\;
	}
	\caption{The HVL-Prime mutation operator} \label{alg:HVLPrime}
\end{algorithm2e}

A single application of HVL-Prime selects uniformly at random one of three suboperations --
insertion, substitution, and deletion -- to be applied at a
location in the solution tree chosen uniformly at random, selecting additional functions or terminals
from the sets $F$ and $L$ of all available functions and terminals as required.
The suboperations are illustrated in Fig.~\ref{fig:HVLPrime}: substitution
can replace any node of the tree with another node chosen uniformly at random
from the set of terminals or the set of functions (if the replaced node is a terminal or a function, respectively),
insertion inserts a new leaf and function node at a random location in the
tree, and deletion can remove a random leaf (replacing its parent with its
sibling).

\begin{figure}[t]
	\centering
	\begin{tabular}{rcl}
	\begin{tikzpicture}[baseline={([yshift=-.5ex]current bounding box.center)},sibling distance=7mm, level distance=7mm, every node/.style={shape=circle, draw, align=center},inner sep=0.6mm,sq/.style={shape=rectangle, minimum height=4mm,rounded corners=2mm,inner sep=1mm}]
		\node[sq] {\GPCBoolOperatorAND}
			child { node[fill=gpcgreen] {$x_2$}}
			child { node[sq] {\GPCBoolOperatorOR} child { node {$x_2$} } child { node {$x_3$}}};
	\end{tikzpicture} &
	$\overset{\texttt{SUB}}{\Leftarrow}$ \hspace{1mm}
	\begin{tikzpicture}[baseline={([yshift=-.5ex]current bounding box.center)},sibling distance=7mm, level distance=7mm, every node/.style={shape=circle, draw, align=center},inner sep=0.6mm,sq/.style={shape=rectangle, minimum height=4mm,rounded corners=2mm,inner sep=1mm}]
		\node[sq] {\GPCBoolOperatorAND}
			child { node[fill=gpcblue] {$x_1$}}
			child { node[sq] {\GPCBoolOperatorOR} child { node {$x_2$} } child { node {$x_3$}}};
	\end{tikzpicture} 
	\hspace{1mm}$\overset{\texttt{INS}}{\Rightarrow}$ &
	\begin{tikzpicture}[baseline={([yshift=-.5ex]current bounding box.center)},sibling distance=7mm, level 1/.style={sibling distance=17mm},level 2/.style={sibling distance=7mm}, level distance=7mm, every node/.style={shape=circle, draw, align=center},inner sep=0.6mm,sq/.style={shape=rectangle, minimum height=4mm,rounded corners=2mm,inner sep=1mm}]
		\node[sq] {\GPCBoolOperatorAND}
			child { node[fill=gpcgreen,sq] {\GPCBoolOperatorOR} child { node[fill=gpcgreen] {$x_3$}} child { node[fill=gpcblue] {$x_1$}} }
			child { node[sq] {\GPCBoolOperatorOR} child { node {$x_2$} } child { node {$x_3$}}};
	\end{tikzpicture} \\[9mm]
	&
	$\Downarrow\,${\scriptsize\texttt{DEL}}
	& \\[1mm]
	&
	\begin{tikzpicture}[baseline={([yshift=-.5ex]current bounding box.center)},sibling distance=7mm, level distance=7mm, every node/.style={shape=circle, draw, align=center},inner sep=0.6mm,sq/.style={shape=rectangle, minimum height=4mm,rounded corners=2mm,inner sep=1mm}]
		\node[sq] {\GPCBoolOperatorOR}
			child { node {$x_2$}}
			child { node {$x_3$}};
	\end{tikzpicture}
	\end{tabular}
	\caption{HVL-Prime suboperations: substitution, insertion, and deletion.} \label{fig:HVLPrime}
\end{figure}
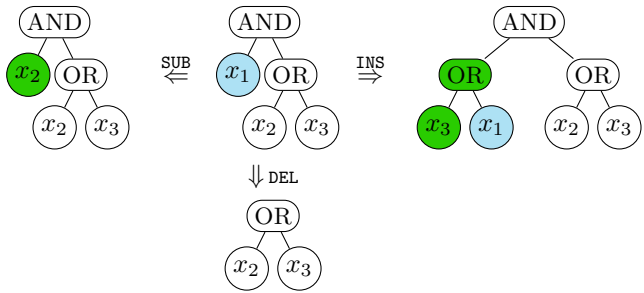

We note that for problems with trivial function
or terminal sets (i.e., those that contain only one element), the substitution
operator is typically restricted to select only from among those nodes which can be
replaced with something other than their current content,
avoiding the situation where the only option is to
substitute a function or terminal node with a copy of itself. This restriction
does not typically affect asymptotic complexity analysis results, as the only
effect of allowing such substitutions is that approximately $1/6$ of the
HVL-Prime applications will not alter the current solution.

In this chapter, we refer to Algorithm~\ref{alg:OneOneGP}, with $k = 1 +
\mathrm{Poisson}(1)$, as the \OneOneGP, differentiating it from the simpler
local search variant which always uses $k=1$, which we call \RLSGP.\footnote{%
In previous work,
the name ``\OneOneGP'' was used for both algorithms, relying either on explicitly
specifying $k$ or on using a suffix as in ``\OneOneGP-multi'' and ``\OneOneGP-single'' to
distinguish between the two variants. Our notation matches the conventions for 
the runtime analysis of evolutionary algorithms \cite{OlivetoX11,Jansen13Book}.}

\OneOneGP algorithms do not use crossover or populations.
Instead, larger changes to the current solution can be performed by multiple
applications of the HVL-Prime operator without evaluating the fitness of
the intermediate trees produced within an iteration. Since each
application of HVL-Prime selects a location in the tree that it will modify
independently, it is possible for this procedure to mutate the parent tree
in several places, rather than only modifying a single subtree (which
would be the case for the standard GP subtree mutation operator, which 
replaces a random subtree of the parent program with a randomly generated
subtree \cite{PoliLMBook}).

\subsection{Bloat Control Mechanisms}
Algorithm~\ref{alg:OneOneGP} depicts the nonstrictly elitist variant of the
\OneOneGP, which accepts offspring as long as they do not decrease the fitness
of the current solution. We use ``\OneOneGPStar'' (and equivalently ``\RLSGPStar'') to
refer to the strictly elitist variant of the algorithm, which only accepts
offspring which have strictly better fitness when compared with the current
solution.

The difference between the elitist and nonelitist variants is often
significant in how the algorithms cope with bloat problems.
The \OneOneGP algorithm operates with a variable-length representation of its
current solution: as mutations are applied, the number of nodes in the tree may
increase or decrease. Poli~et~al. defined bloat as ``program growth
without (significant) return in terms of fitness'' \cite{PoliLMBook}. Bloat can reduce the
effectiveness of GP, as larger programs are potentially more expensive to
evaluate, can be hard to interpret, and may reduce the effectiveness of the GP
operators in exploring the solution space. For example, if a large portion of
the current solution is nonexecutable (perhaps inside an \texttt{if} statement
with a trivially false condition), mutations applied inside that portion of the
program would not alter its behavior, and hence are not helpful in attempting
to improve the program.

Common techniques used to control the impact of bloat include modifying the
genetic operators to produce smaller trees and considering additional
nonfitness-related factors when determining whether an offspring should be accepted
into the population. The latter can include imposing direct limits on the size
of the accepted solutions (by imposing either a maximum tree depth or a maximum tree
size limit), rejecting neutral solutions, or a parsimony pressure approach~\cite{PoliLMBook},
which prefers smaller solutions when the fitness values of two solutions are equal.

Two bloat control approaches that frequently appear in theoretical analyses of GP
algorithms are \emph{lexicographic parsimony pressure} and \emph{Pareto
parsimony pressure} \cite{LukeP02}. The former mechanism breaks ties between equal-fitness
individuals (e.g., in line~7 of Algorithm~\ref{alg:OneOneGP}) by preferring
solutions of smaller size, whereas the latter treats fitness and solution size as
equal objectives in a multiobjective approach to optimization, making
the GP system maintain a population of individuals which do not
Pareto-dominate each other.

\subsection{Evaluating Solution Quality}
In the GP problems analyzed in this chapter, the correct behavior of the target
program is known for all possible inputs. Additionally, in most of the
problems, the GP systems considered are able to evaluate program quality on
all possible inputs efficiently. Both of these
assumptions simplify the analysis, but may not be practical in real-world
applications of GP: the correct output of the target function might only be
known for a limited number of the possible inputs, and/or it might not be practical to
evaluate the candidate solutions for all of the known inputs. Nevertheless,
considering the performance of GP in this setting represents an important first
step: systems which are unable to evolve a program with the desired
behavior using a fitness function which considers all possible inputs are
unlikely to fare better when using a limited approximation. Additionally, fully
deterministic outcomes for solution fitness comparisons simplify the analysis
of the GP systems, allowing their behavior to be described in greater detail.

When the exact fitness is not available, the performance of GP is analyzed in the
PAC learning framework~\cite{Valiant84}. This considers the expected
performance of the GP-evolved program on inputs it may not have encountered
during the optimization process. In this framework, GP evaluates
solution fitness by sampling input/output examples from a training set during
the optimization process, and the goal is to produce a program with a low
generalization error, i.e., with a good probability of producing correct output
on any randomly sampled solution, including ones that have not been sampled
during its construction. The number of samples used
to compare the quality of solutions is an important parameter in this setting,
potentially trading evaluation accuracy for time efficiency.

While a GP algorithm may evaluate solution fitness by relying on a static
training set of polynomial size, for instance chosen at random from the set of
all known inputs/outputs at the start of the optimization process,
Poli~et~{al.} noted that in some circumstances doing so ``may
encourage the population to evolve into a cul-de-sac where it is dominated by
offspring of a single initial program which did well on some fraction of the
training cases, but was unable to fit the others''
\cite[Chapter 10]{PoliLMBook}.
To counteract this when the amount of training set data available is 
sufficient, GP systems can also
opt to compare program quality on samples chosen from the available data for
each comparison \cite{GathercoleR94}. The complexity of these subset selection
algorithms varies from simply selecting inputs/outputs at random (in the case
of random subset selection), through attempting to identify useful inputs/outputs
based on the current or previous GP runs (dynamic or historical subset
selection), to hierarchical combinations of these approaches
\cite{CurryLH07}.

\section{Evolving Tree Structures} \label{sec:gp-tree-structure}
In this section, we review the computational complexity results concerning
the analysis of GP systems for the evolution of trees with specified properties, rather than the evolution of programs with inputs and outputs.
The specific property that the evolved tree should satisfy depends on the problem class.
The possibility of calculating the fitness of candidate solution trees without explicitly executing the program was
regarded as a considerable advantage, since more realistic problems were deemed to be far too difficult for initial computational complexity analyses.

The earliest analysis of the evolution of tree structures considered two
separable problems, called \Order and \Majority. These problems, originally
introduced by Goldberg and O'Reilly \cite{GoldbergO98}, were considered as ``two much simplified, but
still insightful, problems that exhibit a few simple aspects of program
structure'' \cite{DurrettNO11}.
Specifically, \Order and \Majority were introduced as abstracted
simplifications of the eliminative expression that takes place in
conditional statements (where the presence or absence of some element may
\emph{eliminate} others from evaluation, e.g., by making it impossible for
program execution to reach the body of an if statement with an always false
condition), and of the accumulative expression present in many GP applications such
as symbolic regression (where the GP system is able to \emph{accumulate}
information about the correct solution from the aggregate response of a large
number of variables), respectively. In particular, the \Order problem was meant to reflect
conditional programs by making it impossible to express certain variables by
inserting them at certain tree locations (representing portions of the program
which might not ever be executed), while \Majority requires the identification
of the correct set of solution components out of all possible sets. For both
problems the fitness of a candidate solution is determined by an in-order
traversal of its syntax tree.

Neumann additionally introduced weighted variants of the \Order and \Majority
problems. In \WOrder and \WMajority, each pair of variables $x_i, \overline{x}_i$
has a corresponding weight $w_i$, which models the relative importance of the
component to the correctness of the overall solution \cite{Neumann12}.
The idea behind these
weighed variants to mimic the generalization of the complexity analysis of
evolutionary algorithms from \onemax to the class of linear pseudo-Boolean
functions \cite{DrosteJW02,OlivetoX11}.

Another problem considered in the literature where the fitness of solutions depends
on tree structure rather than program execution is \Sorting. In the following
three subsections, we review the state of the art concerning these problems.

The analyses of the toy problems considered in this section have two main aims.
The first is to provide simplified settings that allow rigorous
computational complexity analysis of GP systems by abstracting from the need of
evaluating solution quality on a training set. The second is to evaluate to
what extent bloat affects GP optimization on simplified problems with variable
length representation. Since bloat seems to be a ubiquitous problem in GP, one
expects it to appear also in the optimization process of the problems presented
in this section.

\subsection{The ORDER Problem}
The \Order problem, as originally introduced by Goldberg~and~O'Reilly~\cite{GoldbergO98}, is defined as follows.

\begin{gp-problem}[\Order]
$F := \{J\}$, $L := \{x_1, \overline{x}_1, \ldots, x_n, \overline{x}_n\}$.

The fitness of a tree $X$ is the number of literals $x_i$ for which the
positive literal $x_i$ appears before the negative literal $\overline{x}_i$
in the in-order parse of $X$.
\end{gp-problem}

$J$ (for ``join'') is the only available function in this problem, and the
fitness of a tree is determined by an in-order parse of its leaf nodes; this
reduces the importance of the tree structure in the analysis, making the
representation somewhat similar to a variable-length list.
For example, a tree X with in-order parse $(x_1,\overline{x_4},x_2,\overline{x_1},x_3,\overline{x_6})$ has fitness $f(X)=3$ because $x_1$, $x_2$, and $x_3$
appear before their negations.
Any tree that contains all the positive literals and in which each negative literal $\overline{x_i}$ that appears in the tree is preceded by
the corresponding positive literal $x_i$ has a fitness of $n$ and is optimal.

\Order was introduced as a simple problem that reflects the typical eliminative expressions that take place in conditional statements and other logical 
elements of computer programs, where the presence of an element determines the execution of one program branch rather than another.
The overall idea is that the conditional execution path is determined by inspecting whether a literal or its complement appear first in the in-order leaf parse. 
The task of the GP algorithm is to identify and appropriately position the conditional functions to achieve the correct behavior.

Durrett et al. \cite{DurrettNO11} proved that the (1+1)~GP can optimize \Order in expected time $O(nT_{\max})$, where $T_{\max}$ represents the maximum size 
the evolved tree reaches throughout the optimization process.
The exact result is stated in the following theorem.
\begin{theorem}[\cite{DurrettNO11}]\label{thm:Order}
The expected optimization time of the strictly and nonstrictly elitist cases
of the \RLSGP and \OneOneGP algorithms on \Order is $O(n\,T_{\max})$ in the worst
case, where $n$ is the number of variables $x_i$ and $T_{\max}$ denotes the maximum tree size at any stage during the execution of the algorithm.
\end{theorem}

The proof idea
uses standard fitness-based partition arguments.
Given that at most $k$ variables are expressed correctly (i.e., the positive literal appears before any instances of the corresponding negative literal in the in-order parse of the GP tree), a lower bound of $p_k = \Omega((n-k)^2/(n \max(T, n)))$ may be achieved on the probability of 
expressing an additional literal by an insertion operation given that the GP tree contains exactly $T$ leaf nodes.
Then, by standard waiting-time arguments, the expected number of iterations
required to improve the solution is $1/{p_k}$, and
the expected time until all literals are expressed is
$\sum_{k=1}^n 1/p_k$. 

The runtime bound stated in Theorem~\ref{thm:Order} depends on the tree size $T_{\max}$.
If, as often happens in GP applications, a bound on the maximum size of the tree is imposed, then this bound is also a bound on $T_{\max}$.
However, if no restriction on the maximum tree size is imposed, then bounding the maximum size of the tree is challenging. 
Nevertheless, if strict selection and local mutations are used, then it can be shown that the tree does not grow too much from its initialized size.
The following corollary of Theorem~\ref{thm:Order}, which states this result precisely, is slightly more general than the one presented in \cite{DurrettNO11}.

\begin{corollary}\label{cor:Order}
The expected optimization time of \RLSGPStar on \Order is $O(n^2 + n \, T_{\mathrm{init}})$ if the tree is initialized with
$T_{\mathrm{init}}$ terminals.
\end{corollary}
\begin{proof}
\RLSGPStar will accept only mutations which improve the fitness
of the current solution, and, as there are only $n+1$ possible fitness values,
at most $n$ mutations can be accepted by the GP algorithm before the optimum is found.

A single application of HVL-Prime cannot increase the size of the tree by more
than one leaf. Thus, $T_{\max} \leq T_{\mathrm{init}}+n$, and applying
Theorem~\ref{thm:Order} yields the desired runtime bound.
\hfill~\qed
\end{proof}

It is still an open problem to bound $T_{\max}$ for the \OneOneGP,
or even for \RLSGP where nonstrict selection is used.
It has been conjectured \cite{DurrettNO11} that the same bound as in Corollary~\ref{cor:Order} should also hold for the \OneOneGPStar.
In general, Durrett~et~al.\ noted that the acceptance of neutral moves on
\Order causes a ``feedback loop that stimulates the growth of the tree''
\cite{DurrettNO11}, as there is a slight bias towards accepting insertions
rather than deletions in the problem, and larger trees create more
opportunities for neutral insertions to take place.

A subsequent experimental analysis performed by Urli~et~al. led those authors
to conjecture an $O(T_\mathrm{init} + n \log n)$ upper bound on the runtime
\cite{UrliWN12},
which would imply, if correct, that the bound given in
Corollary~\ref{cor:Order} is not tight.

As shown in the following subsection, by using bloat control mechanisms, more
precise results have been achieved by exploiting more explicit control of
the tree size.

\subsubsection{Bloat Control}

The performance of the \OneOneGP with lexicographic parsimony pressure on \Order
has been considered by Nguyen~et~{al.} \cite{NguyenUW13} and Doerr~et~{al.} \cite{DoerrKLL17}. This mechanism
controls bloat by preferring trees of smaller size when ties amongst
solutions of equal fitness are broken.

Nguyen~et~{al.} used a negative drift theorem to show that as long as the
initial tree is not too large ($T_\mathrm{init} < 19n$), it does not grow
significantly in less than exponential time (i.e., $T_{\max} < 20n$ with high
probability). With this bound on $T_\mathrm{max}$, it was then proven that the
optimum is found in $O(n^2 \log n)$ iterations with high
probability, showing that the solution can be improved up to
$n$ times via a cycle of shrinking it down to minimal size (containing no
redundant copies of any variable) and then expressing a new variable
(pessimistically assuming that this insertion also creates a large number of
redundant terminals in the tree, requiring another round of shrinking to occur
prior to the next insertion).
Experimental results led to the conjecture of an $O(T_\mathrm{init} + n \log n)$ bound \cite{UrliWN12}.

A more precise analysis proves the bound and its
tightness, as given in the following theorem \cite{DoerrKLL17}.

\begin{theorem}[\cite{DoerrKLL17}] \label{thm:ORDER-bloat}
The \OneOneGP with lexicographic parsimony pressure on \Order takes $\Theta(T_\mathrm{init}+n \log n)$ iterations in expectation
to construct the minimal optimal solution.
\end{theorem}

The lower bound of the theorem is proven by using standard coupon collector and
additive drift arguments. For the upper bound, the variable drift theorem
\cite{RoweS12} is applied using a potential function that takes into account
both the number of expressed literals and the size of the tree.

Neumann considered the \emph{Pareto parsimony pressure} approach to bloat
control by introducing a multiobjective GP algorithm (SMO-GP),
and using both the solution fitness and its size as objectives \cite{Neumann12}. 
This approach was motivated by noting that GP practitioners can,
when presented with
a variety of solutions, gain insight into how solution complexity trades off
against quality.

The SMO-GP algorithm
maintains a population of solutions $P$, representing the current best
approximation of the Pareto front. Similarly to the \OneOneGP, the algorithm
produces a single offspring individual by applying the HVL-Prime operator $k$
times to a parent individual chosen uniformly at random from $P$ in each iteration. If
the offspring is not strictly dominated by any solution already in
$P$, it is added to the population, while any solutions in $P$ that it weakly
dominates are removed. Thus, the size of the population $P$ can vary
throughout the run. The theoretical analysis considers the number of
iterations required to compute a population containing the entire Pareto front.

\begin{theorem}[\cite{Neumann12}] \label{thm:SMO-ORDER}
The expected optimization time of SMO-GP, using either $k=1$ or
$k=1+\mathrm{Poisson}(1)$, on \Order is $O(n\,T_{\mathrm{init}}+n^2 \log n)$.
\end{theorem}

The result is proven by showing that it is possible for the GP algorithm to
construct the empty tree in expected $O(nT_\mathrm{init})$ iterations. Once a
minimal solution with $k$ expressed variables exists in the population, the
minimal solution with $k+1$ expressed variables can be constructed from it with
probability at least $\frac{1}{3e}\frac{1}{n+1}\frac{n-k}{2n}$ in each
iteration, and hence an upper bound on the expected runtime may be achieved
by using the fitness-based partition method.

Experiments have led to the unproven conjecture that the bound in Theorem~\ref{thm:SMO-ORDER} is tight \cite{UrliWN12}.

\subsection{The MAJORITY Problem}

The \Majority problem, as originally introduced by Goldberg~and~O'Reilly \cite{GoldbergO98}, is defined as follows.

\begin{gp-problem}[\Majority]
$F := \{J\}$, $L := \{x_1, \overline{x}_1, \ldots, x_n, \overline{x}_n\}$.

The fitness of a tree $X$ is the number of literals $x_i$ for which the positive
literal $x_i$ appears in $X$ at least once, and at least as many times as the
corresponding negative literal $\overline{x}_i$.
\end{gp-problem}

$J$ (for ``join'') is the only available function in this problem, and the
fitness of a tree is determined by an in-order parse of its leaf nodes; this
reduces the importance of the tree structure in the analysis, making the
representation somewhat similar to a variable-length list.
For example, a tree with an in-order parse of $(\overline{x}_1, x_1, x_2, x_3,
\overline{x}_3, \overline{x}_3)$ would have a fitness of $2$, as only the literals $x_1$
and $x_2$ are expressed (while $\overline{x}_3$ outnumbers $x_3$ in
the tree, and $x_3$ is therefore suppressed). Any optimal solution, expressing
all $n$ positive literals, has a fitness of $n$.

The fitness of solutions in \Majority is based on the number of literals $x_i$ 
and $\overline{x}_i$ in the tree, with only the literal in greater
quantity (the majority) being expressed and potentially contributing to the fitness
value. This serves to model problems where solution fitness can be accumulated
through additions of more nodes to the tree, regardless of their exact
positions.

In contrast to \Order, where there is always a position in the tree where an
unexpressed literal $x_i$ can be inserted to express $x_i$ and improve the
fitness of a solution, in \Majority there exist trees where no single insertion
of an unexpressed $x_i$ will lead to $x_i$ being expressed and thus improving
the fitness, even though all literals $x_i$ can contribute to expressing $x_i$ in
aggregate regardless of their position.
Thus, GP variants which do
not accept neutral moves have been found to perform quite badly,
with \RLSGPStar shown to be capable of getting stuck in
easily constructed local optima, and \OneOneGPStar having an exponential
expected optimization time to recover from a worst-case initialization
\cite{DurrettNO11}.
On the other hand, GP variants using nonstrict selection may be efficient.

\begin{theorem}[\cite{DurrettNO11}]\label{thm:Majority-Single}
Let $T_{\max}$ denote the maximum tree size at any stage
during the execution of the algorithm. Then the expected optimization time of \RLSGP on \Majority is
$$ O(n \log n + DT_{\max} n \log \log n) $$
in the worst case, where $D := \max(0, \max_i(c(\overline{x}_i) -
c(x_i)))$ and $c(x)$ is the number of times the literal $x$ appears in the
initial tree.

If the algorithm is initialized with a random tree containing $2n$ terminals selected uniformly at random from $L$, the expected optimization time of \RLSGP on \Majority is
$ O(n^2 T_{\max} \log \log n) .$
\end{theorem}

The bounds presented depend on $D$, the maximum deficit between the numbers of
positive literals and negative literals of any variable in the tree (thus, a
tree with a single copy of $x_1$ and two copies of $\overline{x}_1$ would have
a deficit $D=1$). The worst-case result, assuming a deficit of $D$
literals for all $n$ variables, follows from a generalized variant of the
coupon collector problem \cite{MyersW03}, requiring the collection of $D$
copies of each coupon. For a uniform initialization with
$T_{\mathrm{init}} = 2n$, a bound $D = O(\log n / \log \log n)$ was derived
using the balls-into-bins model \cite{MitzenmacherUpfalBook}. It was then
proven that a variable which initially has a deficit of $D$ becomes
expressed after an expected $O(DT_\mathrm{max})$ mutations involving that
variable (which occur with probability $\Theta(1/n)$) by showing that the
GP system essentially performs a random walk that is at least fair with respect
to decreasing the deficit.

For the \OneOneGP, only a hypothetical worst-case
analysis for the elitist variant was presented in \cite{DurrettNO11},
noting that if the last unexpressed
variable has $k$ more negative literals than positive literals in the tree, the
final mutation will require at least $\Omega(n^{k/2})$ time, and thus, unless
$k$ can be shown to be constant, the expected runtime remains superpolynomial.
However, no bounds on the probability that a superconstant $k$ would actually
occur were given.

The problem, including the
dependence on $T_{\max}$ was recently solved, proving
the following upper and lower bounds on the expected optimization time \cite{DoerrKLL17}.

\begin{theorem}[\cite{DoerrKLL17}] \label{thm:MAJORITY-nobloatcontrol}
When the algorithm is initialized with a tree containing $T_\mathrm{init}$ terminals,
the expected optimization time of the \RLSGP and \OneOneGP algorithms on \Majority
is at least $ \Omega(T_\mathrm{init} + n \log n)$ and
at most $O(T_\mathrm{init}\log T_\mathrm{init} + n \log^3 n).$
\end{theorem}

The lower bound is proven by an application of the multiplicative drift theorem
with bounded step size, while the upper bound relies on showing that if
$T_\mathrm{init} \geq n \log^2 n$, the tree will grow by at most a constant
factor in $O(T_{\mathrm{init}}\log T_{\mathrm{init}})$ generations
before the optimal solution is constructed. As a result, bloat 
does not hinder the optimization process, i.e., the final tree
may be at most larger by a multiplicative polylogarithmic factor than the
optimal solution size.

From the analysis, an interesting alternative to bloat control emerges. If
the HVL mutation probabilities were changed such that deletions were more likely
than insertions, a drift towards smaller solutions would be observed, leading
to smaller trees, and hence faster optimization. Such a suggestion was
originally made by Durrett~et~{al.}, albeit for the \Order problem  \cite{DurrettNO11}.
Concerning \Majority, theoretical evidence in support of this has emerged,
though no formal proof is available \cite{DoerrKLL17}.

\subsubsection{Bloat Control}
Applying lexicographic parsimony pressure mitigates the analysis problems
that arise with GP systems for \Majority. With this bloat control mechanism, mutations which
solely remove negated terminals are always accepted, as they reduce the size of
the tree. Accepting such mutations eventually leads the GP system
to a solution where
fitness can be improved by inserting a positive literal, allowing the optimum
to be reached efficiently.

\begin{theorem}[\cite{Neumann12}] \label{thm:MAJORITY-bloat-single}
The expected optimization time of \RLSGP with lexicographic parsimony
pressure on \Majority, when initialized with a tree containing
$T_\mathrm{init}$ literals, is
$ O(T_\mathrm{init} + n\log n). $
\end{theorem}

The result is proven by reasoning that it takes $O(T_\mathrm{init})$ iterations
to remove the $T_\mathrm{init}$ negated terminals provided by a worst-case
initialization, and $O(n \log n)$ iterations to express all $n$ variables
by an application of the coupon collector argument.

A tight bound for the \OneOneGP, showing that
the larger Poisson mutations do not affect the asymptotic runtime,
has recently been proven \cite{DoerrKLL17}, confirming
a previous conjecture~\cite{UrliWN12}.

\begin{theorem}[\cite{DoerrKLL17}] \label{thm:MAJORITY-bloat-multi}
The expected optimization time of the \OneOneGP with lexicographic parsimony
pressure on \Majority, when initialized with a tree containing
$T_\mathrm{init}$ literals, is
$\Theta(T_\mathrm{init} + n\log n)$.
\end{theorem}

The lower bound of the theorem is proven by using standard coupon collector and
additive drift arguments. For the upper bound, the variable drift theorem
\cite{RoweS12} is applied using a potential function that takes into account
both the number of expressed literals and the size of the tree. Intuitively, the size of the tree is only allowed to increase if
the \Majority fitness is also increased, which can only occur a limited number
of times, and the magnitude of the increase is unlikely to be overly large owing
to the Poisson distribution used to determine $k$.

It is still an open problem to prove that lexicographic parsimony pressure
asymptotically improves the runtime of the \OneOneGP or that the upper bound
given in Theorem~\ref{thm:MAJORITY-nobloatcontrol} is not tight
(Urli~et~{al.} conjectured an upper bound of $O(T_\mathrm{init} + n \log n)$
without bloat control, based on experimental data \cite{UrliWN12}).

Applying Pareto parsimony pressure and treating the size of the tree as an
additional objective in the multiobjective SMO-GP algorithm allows the GP
system to compute the Pareto front of solutions in terms of fitness/complexity.

\begin{theorem}[\cite{Neumann12}]
The expected optimization time of SMO-GP (with either $k=1$ or
$k=1+\mathrm{Poisson}(1)$) on \Majority, initialized with a single tree
containing $T_\mathrm{init}$ terminals, is $O(n T_\mathrm{init} + n^2 \log n)$.
\end{theorem}

The SMO-GP population will contain at most $n+1$ individuals, as there are only
$n+1$ distinct fitness values for \Majority. Similarly to the situation for
lexicographic parsimony pressure, SMO-GP is able to construct an initial
solution on the Pareto front by repeatedly removing any duplicate or negated
terminals from the initial solution. Once a solution on the Pareto front
exists, the entire front can be constructed by repeatedly selecting a solution
at the edge of the front and expressing an additional variable or deleting an
expressed variable.

\subsubsection{More Complex MAJORITY Variants}

Given that the \Majority problem can be efficiently optimized by simple GP
systems without bloat appearing as a problem, more sophisticated versions of
the problem have been designed \cite{KotzingLLM18}.

In the $+c$-\Majority problem, $x_i$ is expressed if and only if the
number of $x_i$ literals in the tree exceeds the number of $\overline{x}_i$
literals by at least $c$. It has been proven that the \RLSGP is with high probability
not able to find the optimal solution when $c>1$ and lexicographic parsimony
pressure is employed, but is able to do so in expected polynomial time when no
bloat control mechanism is used. In this problem, the impact of bloat is
limited, as the insertions of $x_i$ and $\overline{x}_i$ are accepted with
equal probability when $x_i$ is not expressed, and the necessary margin to
express $x_i$ can be reached as a consequence of a fair random walk. On the
other hand, lexicographic parsimony pressure prevents this random walk from
taking place, as only mutations which increase the number of expressed
variables or reduce the size of the tree would be accepted. Thus, \RLSGP
with lexicographic parsimony pressure cannot express $x_i$ unless at least
$c-1$ copies of $x_i$ are already present in the initial solution.

The opposite holds for the $2/3$-\textsc{SuperMajority} problem, which provides
a fitness reward of $2-2^{c(\overline{x}_i)-c(x_i)}$ for every variable $x_i$
for which $c(x_i) > 2 c(\overline{x_i})$, where $c(z)$ denotes the number of
times the literal $z$ appears in the tree. In particular, the \RLSGP without
bloat control is with high probability not able to express all $n$ variables,
and thus cannot find solutions with fitness above a certain threshold.

\begin{theorem}[\cite{KotzingLLM18}]
For any constant $\nu > 0$, consider the \RLSGP without bloat control on
$2/3$-\textsc{SuperMajority} on the initial tree with size $s_\mathrm{init} =
\nu n$. There is $\varepsilon = \varepsilon(\nu) > 0$ such that, with
probability $1-o(1)$, an $\varepsilon$-fraction of the variables will never be
expressed. In particular, the algorithm will never reach a fitness larger than
$(2-2\varepsilon)n$.
\end{theorem}

The proof idea relies on showing that the size of the current solution
increases over time (due to the fitness rewards for inserting additional copies
of positive literals for expressed variables), which makes insertions of
non-expressed variables more likely to occur than their deletions. This makes
reaching the $2/3$-majority threshold to express a variable difficult,
requiring a significant deviation from the expected outcome of a fair random
process. Lexicographic parsimony pressure, when employed, sidesteps this problem
by gradually removing literals of non-expressed variables from the tree, and
eventually allowing $x_i$ to be expressed by a single insertion of its positive
literal.

K\"otzing~et~{al.} additionally proved that a memetic GP algorithm with a
simple concatenation crossover mechanism and local search to remove redundant
literals is able to efficiently solve both the $+c$-\Majority and
$2/3$-\textsc{SuperMajority} problems \cite{KotzingLLM18} if lexicographic
parsimony pressure is employed. Hence they provide an example where
incorporating a population and applying crossover allows a wider range of
problems to be solved.

\subsection{The SORTING Problem}

The \Sorting problem is the first classical combinatorial optimization problem
for which computational complexity
results have been obtained for discrete evolutionary algorithms. For the
application of evolutionary algorithms Scharnow~et~{al.} defined \Sorting
as the problem of maximizing different measures of sortedness of
a permutation of a totally ordered set of elements \cite{ScharnowTW04}.

Wagner~et~{al.} analyzed the performance of GP for the problem,
aiming to investigate the differences between different
bloat control mechanisms for GP \cite{WagnerN12,WagnerNU15}.
For GP, the
measures of sortedness were adapted to deal with incomplete permutations
of the literal set.

\begin{gp-problem}[\Sorting]
$F := \{J\}$, $L := \{1, 2, \ldots, n\}$. 

The fitness of a tree $X$ is computed
by deriving a sequence $\pi$ of symbols based on their first appearance in the
in-order parse of $X$, and considering one of the following
five measures of sortedness of this sequence.

\begin{tabular}{rp{98mm}}
	{$\mathrm{INV}(\pi)$} & Number of pairs of adjacent elements in the correct order (maximize to sort), with $\mathrm{INV}(\pi) = 0.5$ if $|\pi| = 1$.
\\	{$\mathrm{HAM}(\pi)$} & Number of elements in correct position (maximize to sort).
\\	{$\mathrm{RUN}(\pi)$} & Number of maximal sorted blocks (minimize to sort), plus the number of missing elements $n - |\pi|$, with $\mathrm{RUN}(\pi) = n+1$ if $|\pi| = 0$.
\\	{$\mathrm{LAS}(\pi)$} & Length of longest ascending sequence (maximize to sort).
\\	{$\mathrm{EXC}(\pi)$} & Smallest number of exchanges needed to sort the sequence (minimize to sort), plus $1 + n - |\pi|$ if $|\pi| < n$.
\end{tabular}
\end{gp-problem}

$J$ (for ``join'') is the only available function in this problem, and the
fitness of a tree is determined by an in-order parse of its leaf nodes drawn
from a totally ordered set of terminals $L$. This reduces the importance of the
tree structure in the analysis, making the representation somewhat similar to a
variable-length list.
Thus, for $n=5$, the fitness of a tree with an in-order parse of $(1, 2, 1, 4,
5, 4, 3)$, and hence $\pi = (1, 2, 4, 5, 3)$ is INV$(\pi) = 3$, HAM$(\pi) =
2$, RUN$(\pi) = 2$, LAS$(\pi) = 4$, and EXC$(\pi) = 2$. The fitness value of
optimal trees for the INV, HAM, and LAS measures is $n$, while for the RUN and
EXC measures it is $0$.

Unlike the \Order and \Majority problems considered in the previous sections, the \Sorting
problem is not separable, meaning that it cannot be split into subproblems that
could be solved independently. The dependencies between the subproblems can
thus significantly impact the overall time needed to solve the optimization
problem, and the variable-length representation of solutions can create local
optima from which it is difficult for GP systems to escape.
Wagner~et~{al.} additionally remarked that the task of evolving a solution is
more difficult for the \RLSGP and \OneOneGP systems considered than for the
permutation-based EA, which in expectation requires $O(n^2 \log n)$
iterations for the INV, HAM, LAS, or EXC sortedness measure, and exponential
time when using the RUN sortedness measure \cite{ScharnowTW04}.

\begin{theorem}[\cite{WagnerNU15}] \label{thm:gp-SORTING-INV-single}
The expected optimization time for the \RLSGPStar and \OneOneGPStar algorithms
on \Sorting using $\mathrm{INV}$ as the sortedness measure is $O(n^3 T_{\max})$, where $n$
is the number of elements to be sorted, and $T_{\max}$ is the maximum size of
the tree during the run of the algorithm.

For the $\mathrm{HAM}$, $\mathrm{RUN}$, $\mathrm{LAS}$, and $\mathrm{EXC}$ measures, there exist initial solutions with
$O(n)$ terminals such that the expected optimization time of \RLSGPStar is
infinite, and the expected optimization time of \OneOneGPStar is
$e^{\Omega(n)}$.
\end{theorem}

The positive statement is proven by applying the artificial fitness level method, observing that there are
$n\cdot(n-1)/2+1$ possible fitness values, and with probability
$\Omega(1/(nT_{\max}))$ a mutation inserts a literal which corrects at least
one unsorted pair without introducing any additional unsorted pairs.

For the HAM, RUN, LAS, and EXC measures, trees which require large mutations to
improve fitness exist, which causes the expected optimization time to be
infinite for \RLSGPStar and $e^{\Omega(n)}$ for the \OneOneGPStar. In
general, the problematic solutions contain a large number of copies of a single
literal in an incorrect location and a large sorted sequence, requiring either
all the incorrectly placed copies to be removed simultaneously or the sorted
sequence to be moved in a single mutation.

\subsubsection{Bloat Control}
When bloat control mechanisms are applied,
GP systems may reduce the size of the redundant components of the solution even
if mutations which make progress in this direction do not alter the solution's
sortedness measure.

The impact of applying lexicographic parsimony pressure for the \OneOneGP
family of algorithms and of Pareto parsimony pressure for the SMO-GP
algorithms has been considered \cite{WagnerN12,WagnerNU15}. We
summarize the results in Table~\ref{tab:gp-SORTING}.

\begin{table}[t]
\caption{Known expected runtimes for GP algorithms on \Sorting using various sortedness measures and bloat control mechanisms.}
\label{tab:gp-SORTING}
\centering
\begin{tabular}{r|cc|ccc} \hline
& \multicolumn{2}{c|}{No bloat control} & \multicolumn{2}{c}{Parsimony pressure} \\
$F(X)$ & \RLSGPStar & \OneOneGPStar & \RLSGP & SMO-GP
\\ \hline
   \rule{0pt}{2.4ex} INV & $O(n^3 T_{\max})^a$ & $O(n^3 T_{\max})^a$ &
	$O(T_\mathrm{init} + n^5)^a$ & $O(n^2 T_\mathrm{init} + n^5)^a$
\\ \rule{0pt}{2.5ex} \rule[-1.3ex]{0pt}{0pt}
   LAS & $\infty^a$ & $\Omega\left(\left(\frac{n}{e}\right)^n\right)^a$ &
	$O(T_\mathrm{init} + n^2\log n)^{a,b}$ & $O(n T_\mathrm{init} + n^3 \log n)^a$
\\ \rule{0pt}{2.5ex} \rule[-1.3ex]{0pt}{0pt}
   HAM & $\infty^a$ & $\Omega\left(\left(\frac{n}{e}\right)^n\right)^a$ &
	$\infty^c$ & $O(n T_\mathrm{init} + n^4)^c$
\\ \rule{0pt}{2.5ex} \rule[-1.3ex]{0pt}{0pt}
   EXC & $\infty^a$ & $\Omega\left(\left(\frac{n}{e}\right)^n\right)^a$ &
	$\infty^c$ & $O(n^2 T_\mathrm{init} + n^3\log n)^c$
\\ \rule{0pt}{2.5ex} \rule[-1.3ex]{0pt}{0pt}
   RUN & $\infty^a$ & $\Omega\left(\left(\frac{n}{e}\right)^n\right) ^a$ &
	$\infty^c$ & $O(n^2 T_\mathrm{init} + n^3\log n)^c$ \\ \hline
\end{tabular}
\parbox[l]{10cm}{ \footnotesize
$^a$ Shown in \cite{WagnerNU15}. \\
$^b$ Also holds with probability $1-o(1)$ for the \OneOneGP. \\
$^c$ Shown in \cite{WagnerN12}.
}
\end{table}

In general, the positive results are proven by
showing that there exists a sequence of fitness-improving mutations leading the
GP system to the global optimum (in the case of \OneOneGP algorithms), or, for
SMO-GP, to a solution on the Pareto front from which other Pareto front
solutions can be constructed efficiently.

The majority of the negative results rely on showing the existence of local
optima for the sortedness measure, which limits the availability of results for
nonstrictly elitist algorithms, and especially for the \OneOneGP, which is
capable of performing larger mutations.

The results in Table \ref{tab:gp-SORTING} suggest that the
variable-length representation can cause difficulties for \RLSGP even when
parsimony pressure is applied, for some simple measures of sortedness, while
even a simple multiobjective algorithm is able to find the entire Pareto front
of the problem efficiently when using any of the five measures considered.

Experimental results have been presented that suggest that the \OneOneGP
algorithm is efficient (i.e., able to find the optimum in polynomial time) using
all of the sortedness measures considered except RUN, both with and without
bloat control mechanisms: concerning the average-case complexity, an $O(n^2
\log n)$ bound has been conjectured for the INV and LAS measures, and an $O(n^4)$ bound for the EXC and HAM measures \cite{WagnerNU15}.
Providing a rigorous theoretical
analysis of the behavior of these GP systems remains an open question.

\subsection{Outlook}

In this section, we have provided an overview of the computational complexity
results for simple GP systems for toy problems where the evolved GP trees may
grow to arbitrarily large sizes. The main aim behind the analyses is to shed
light on how bloat affects the optimization process of GP. Surprisingly, bloat
does not hinder the efficient optimization of the \OneOneGP for any of the
basic problems. Theorem~\ref{thm:MAJORITY-nobloatcontrol} provides a rigorous
proof of this for \Majority, while experimental work has lead to
similar conjectures for \Order and \Sorting, although formal proofs are not
yet available.

Recently, a toy problem has been designed where the \RLSGP provably
requires exponential time with overwhelming probability due to bloat. To
achieve this result, the design of 2/3-\textsc{SuperMajority} closely follows
the definition of ``bloat''. Indeed, fitness increases slightly with the
increase of the tree size, making it less and less likely that significantly
beneficial mutations occur. Nevertheless, simple bloat control mechanisms, such
as lexicographic parsimony pressure, effectively address the issue. Thus they
allow the \RLSGP to efficiently optimize 2/3-\textsc{SuperMajority}. Overall,
there is still a need to design benchmark functions that reflect the reported
behavior of GP in practice, i.e., problems where bloat occurs and are difficult
to solve with the use of bloat control techniques.

\section{Evolving Programs of Fixed Size} \label{sec:gp-towards-programs}

In this section, we consider two more advanced applications compared with those
in the previous section. For both problems, the fitness of an evolved program
is computed by evaluating its output. While more realistic, these problems are
still different from real-world GP applications. In the first problem, MAX, the
program to be evolved has no input variables, and thus the GP system has to construct
a program which always outputs the same constant value, subject to constraints
on problem size and available operators. Concerning the second problem, 
\IdentificationP, the structure of the optimal solution is fixed (i.e., no
tree structure has to be evolved), and the GP system is not allowed
to deviate from it, but must instead learn the exact weights of a predefined
linear function while evaluating program quality by comparing the program
output with the target function on only a limited number of the
possible function inputs.

The first toy problem, MAX, may reflect practical GP applications where bloat
is avoided by setting a maximum limit on the size or height of the evolved
trees. When such a limit is reached, large tree modifications may be required
to make further progress. Such a problem occurs, for example, for GP evolving
Boolean conjunctions with a function set comprising of AND and OR (see
Theorem\ref{thm:gp-conj-andor-runtime} in
Section~\ref{sec:gp-conj-expressive}). The second problem, \IdentificationP,
models the issue that the true fitness of candidate solutions in GP is usually
unknown, and their quality has to be estimated using a training set.

\subsection{The MAX Problem} \label{sec:gp-max}

The MAX problem was originally introduced by Gathercole~and~Ross as a means of
analyzing the limitations of crossover when applied to trees of fixed
size \cite{GathercoleR96}. The fitness of the program depends on the evaluation of the arithmetic
expression represented by the tree. However, the problem
contains no variable inputs, and thus the goal of the GP algorithm is simply to
construct a tree that evaluates to the maximum possible value subject to
restrictions on the size of the tree, and on the available functions and
terminals.

\begin{gp-problem}[MAX]
$F := \{+, \times\}$, $L := \{t\}$, $t > 0$ a positive constant, and maximum tree depth $D$.

The fitness of a tree $X$ is the value produced by evaluating
the arithmetic expression represented by the tree
if the tree is of depth at most $D$, and 0 if the tree is of larger depth.
\end{gp-problem}

The optimal solution to MAX is a complete binary tree of depth $D$, with $t$ at
all the leaf nodes, and with the lowest $\lfloor 1/2+1/t \rfloor$ levels of internal (i.e., nonleaf)
nodes containing $+$ and the remaining internal nodes containing $\times$.
It has been noted that lower values of $t < 1$ make the problem more
difficult for crossover-based GP systems~\cite{GathercoleR96}.

The behavior of GP systems on the MAX problem was previously studied
experimentally, with Langdon~and~Poli observing that MAX is hard for GP
systems utilizing crossover owing to the interaction of deception with the depth
bound on the tree making it difficult to evolve solutions. The GP system is
essentially forced to perform randomized hill climbing in the later
stages of the optimization process, and hence requires exponential time with
respect to the maximum allowed depth of the tree \cite{LangdonPBook}.

A theoretical analysis of the \OneOneGP for the MAX problem was presented by
K\"otzing~et~{al.} \cite{KotzingSNO14}, who proved that the runtime of the mutation-only algorithm
is polynomial with respect to $n = 2^{D+1}-1$, the maximum allowed number of
nodes in the tree.

\begin{theorem}[\cite{KotzingSNO14}]
The \RLSGP algorithm finds the optimal solution for the MAX problem for any
choice of $t > 0$, in expected $O(n \log n)$ iterations, where $n$ is the
maximum allowed number of nodes in a tree subject to the depth limit $D$.
\end{theorem}

The theorem is proven by showing that the GP algorithm can first construct a complete
binary tree with depth $D$ in a way that prevents any node from being deleted,
and then use the substitution suboperation of HVL-Prime to correct internal
nodes.

Concerning the \OneOneGP, a weaker bound on the expected runtime was proven.

\begin{theorem}[\cite{KotzingSNO14}]
The expected time for the \OneOneGP to find the optimal solution for the MAX
problem with $t=1$ is $O(n^2)$.
\end{theorem}

The theorem is proven using fitness-based partitions, exploiting the existence
of at least one leaf in a tree of size $n$ which could be selected by
insertion to grow the tree. Experimental results suggesting that the true
runtime of the \OneOneGP on MAX is also $O(n \log n)$ were also presented, and
the authors of \cite{KotzingSNO14} noted that a more precise potential function
based on the contents
of the tree would be required to show this upper bound using drift analysis.

Additionally, a modification of the insertion operation in HVL-Prime to grow
the tree in a more balanced fashion was considered: rather than selecting a
location to insert a new leaf node uniformly at random from the entire tree,
selection would pick a leaf at depth $d$ with probability $2^{-d}$ to be
replaced with a new function node, using the original leaf and an inserted
terminal as its children. As well as balancing the growth of the tree between
different branches, this reduces the probability that mutation attempts
insertion operations which would be blocked by the tree depth limit. With this
modified insertion operator, an $O(n \log n)$ bound
on the expected runtime of the \OneOneGP on MAX with $F=\{+\}$ was proven~\cite{KotzingSNO14}.

Closing the gap between the $O(n^2)$ upper bound for the \OneOneGP on MAX with
$F=\{+, \times\}$ and the $\Omega(n \log n)$ lower bound given by a coupon
collector argument remains an open problem. Furthermore, theoretical time
complexity analyses of the performance of crossover-based GP systems, for which
the MAX problem was originally introduced, are still unavailable.

\subsection{The Identification Problem and PAC Learning} \label{sec:GP-PAC}

It is generally not possible to evaluate the quality of the evolved programs
on \emph{all} possible inputs efficiently, as they usually are too numerous when the
number or the domain of input variables is too large. 
The \IdentificationP problem was introduced by
K\"otzing~et~al.~\cite{KotzingNS11} to evaluate the learning capabilities of a
simple evolutionary algorithm, an EA with a local mutation operator that 
evaluates program quality by considering only a polynomial number of inputs 
chosen uniformly at random in each iteration. This setting is the same as
that of the PAC learning framework \cite{Valiant84}.
The idea is that while some problems cannot always be solved exactly (as there
might be no known polynomial-time algorithm that produces an exact solution, 
as, e.g., for NP-hard problems), a good
approximation, i.e., one that is correct on a random input with high
probability, may be achieved. A large class of functions
has been shown to be PAC learnable by designing appropriate evolutionary
algorithms \cite{Valiant09,Feldman12}.
Compared with those studies,
K\"otzing~et~{al.} considered a simplified setting \cite{KotzingNS11}.
Unlike the problems previously considered, the structure of the desired
solution is known in advance by the algorithm, which has to identify
the target function among a known class of linear functions.
More precisely, the \IdentificationP problem is to learn the
weights of a linear function
$f_\mathrm{OPT}$ defined over bit strings $x \in \{0,1\}^n$,
$$ f_\mathrm{OPT}(x) = \sum_{i=1}^n w_i x_i, $$
where $w_i \in \{-1, 1\}$.

The goal of the EA (called the Linear~GP algorithm)
is to identify whether each weight $w_i$ is positive or
negative. The algorithm changes a single weight $w_i$ in each iteration, and
determines whether the mutated offspring has better fitness than its parent
using a multiset $S$ constructed independently in each iteration by selecting
the desired number of points uniformly at random (with replacement) from $\{0, 1\}^n$. The error $e_S$ of each solution $f$ is computed as
$$ e_S(f, f_\mathrm{OPT}) = \sum_{x\in S} |f(x)-f_\mathrm{OPT}(x)|, $$
and solutions with lower error are preferred.

Thus, the focus of the analysis is to measure the ability of the GP system to extract
information from a limited view of the true fitness function: if $S$ is too
small, the sampled error function may be an unreliable indication of the true
quality of the solution. On the other hand, if $S$ is too large, more
computational effort than necessary is expended for each fitness evaluation, 
which could result in worse performance with respect to the overall CPU time 
spent.

The following theorem shows that the Linear~GP algorithm is able to learn
$f_\mathrm{OPT}$ efficiently if the number of inputs sampled in each iteration is
sufficiently large.

\begin{theorem}[\cite{KotzingNS11}]
If $|S| \geq c_0 n \log n$, $c_0$ a large enough constant, the expected number
of generations until the best-so-far function found by Linear~GP has an
expected error $\leq \delta$ is $O(n \log n + n^2/\delta^2)$.

If $f_\mathrm{OPT}$ also has a linear number of both $1$ and $-1$ weights, the
expected number of generations until such a solution is found is $O(n +
n^2/\delta^2)$.
\end{theorem}

In this setting, $e_S \leq 1$ implies that an optimal solution has been found,
and thus the theorem additionally provides an $O(n^2)$ bound on the expected
number of generations required to learn $f_\mathrm{OPT}$ perfectly (by setting
$\delta = 1$). The theorem is proven
by showing that in $O(n \log n)$ generations, the numbers $c_1$ and $c_{-1}$ of incorrect weights
in $f$ set to $1$ and $-1$, respectively, become balanced (such that there is at most one
more incorrect weight of one kind than the other) with high probability, and
remain balanced throughout the rest of the process. When $c_1 = c_{-1}$, mutations that increase either value are
rejected with high probability, while mutations that reduce either value are accepted
with high probability (but can be undone by the GP system until a wrong weight of the
opposite kind is corrected). Thus, $c_1$ and $c_{-1}$ can be reduced
permanently by performing the two reductions in sequence (which occurs with
probability at least $(i/n)^2$ if, initially, $c_1 = c_{-1} = i$), and, by a coupon collector-like argument, the number of
incorrect weights is reduced to an acceptable level in expectation after
$O(n^2/\delta^2)$ generations.

Extending the analysis to broader function classes and algorithms, for example
considering functions with more than two options for each coefficient, or a
\OneOneGP-like mutation operator capable of performing more than one change in
each iteration, remains an open direction for further research. The
PAC learning framework will also be used to analyze the performance of the
\OneOneGP family of algorithms on Boolean functions in the next section.

\subsection{Outlook}

The MAX problem is easy for mutation-based GP systems. Yet, the achievement of
precise asymptotic bounds on their runtime is still prohibitive. On the other
hand, the crossover-based GP algorithms used in practice do not achieve a
significant benefit from crossover on MAX \cite{GathercoleR96}. How this could
be rigorously proven remains an open problem.

Small super-linear polynomial size training sets suffice to efficiently
estimate the true fitness of candidate solutions for linear functions with
\{1,-1\} weights. This allows the exact identification of the target function
of the \IdentificationP problem. Generalization of this result to larger weight
sets and function classes would support future analyses of realistic symbolic
regression applications.

\section{Evolving Proper Programs: Boolean Functions} \label{sec:gp-boolean}

In real-world applications of GP systems, the goal is to evolve a program with
specific behavior. In most applications, the program accepts some inputs and
produces one or more output values, and the quality of candidate programs
is evaluated by executing them on a variety of possible inputs for
which the correct output is known. The structure of the target program is
typically not known in advance, and thus the GP systems may be given access to
more components (both functions and terminals) than is strictly necessary to
represent an optimal solution. Real-world applications of GP can exhibit all the
challenges that the previously discussed problems modeled in isolation: the
length and structure of the target program are not known to the GP system,
there may be a variety of function and terminal nodes, and solution quality is
evaluated by executing the program on some or all of the possible inputs.

Boolean functions, which take a number of binary inputs and produce a single
binary output, have long been used as benchmarks in the field of GP
\cite{KozaBook92,LangdonPBook} and are a natural next step for the complexity
analysis of GP systems, as they can combine all of these challenges. The
problems of evolving some Boolean functions, such as conjunctions
(AND) or parity (XOR), are also well understood in the PAC learning framework
\cite{Valiant09} -- conjunctions are evolvable efficiently, while parity
problems are not. Additionally, such problems form an interesting sanity check
for the (1+1)~GP algorithms: if the simple algorithms are not able to evolve
relatively simple functions, it would be interesting to determine which
components of the more complex GP algorithms enable these problems to be solved
efficiently, i.e., to identify how much sophistication is required in the
GP system for it to be efficient.

A complexity analysis of (1+1) GP algorithms for the AND and XOR problems, where
the goal is to construct a conjunction or an even parity function,
has recently been presented \cite{MambriniO16}.
For these problems, the fitness of the evolved solutions was evaluated by
comparing their output with that of the target function on either the entire
truth table or a polynomial training subset.

Using the complete truth table (i.e., all possible inputs) as the training set
is typically only feasible for Boolean functions if the size of the problem, in
terms of the number of input variables, is relatively small (as there are $2^n$
possible inputs for $n$ Boolean input variables, and evaluating each candidate
program on an exponential number of inputs would require exponential time).
However, benchmark problems with small $n$ have been considered for GP systems,
and may still occur in some settings. Additionally, a confirmation of whether a
given GP system can evolve a given function given an exact fitness function
(i.e., the complete truth table) is
also useful for further analysis: if it cannot, it is likely that mechanisms
more complex than random sampling of inputs would be required to evolve the
function in polynomial time.

If an incomplete training set is used, the GP system may either choose it once
at the beginning of the run (the static incomplete training set case, as
considered in \cite{MambriniO16}), or choose a fresh subset dynamically in
every iteration (as in \cite{LissovoiO18}). Both approaches may be valid in
different practical settings. If the complete truth table is known but is
prohibitively large, it may be sampled to estimate the fitness of a solution,
reducing the computational effort required to evaluate the quality of a program
at the cost of introducing some uncertainty. On the other hand, if only a
limited number of input/output examples are available, some may need to be
reserved to validate the quality of the solution on inputs that it has not been
trained on.

\subsection{Evolving Conjunctions}
For the AND problem, the target function that the GP system has to evolve is a
conjunction of some number of variables. Conjunctions have an easy to
understand input-to-output mapping simplifying the analysis, and are known to be
efficiently evolvable~\cite{Valiant09}. However, unlike tailored learning
algorithms, the GP systems do not necessarily know that the target function is
a conjunction -- and ideally, should be able to evolve conjunctions even with
access to a variety of functions and terminals.

\begin{gp-problem}[AND]
Let $L \subseteq \{x_1, \ldots, x_n\}$ be the set of available terminals, and
$F$ be the set of available functions.

The fitness of a tree $X$ using a training set $T$ selected from the rows of
the complete truth table $C$ is the number of training set rows on which the
value produced by evaluating the Boolean expression represented by the tree
differs from the output of the target function: the conjunction of all (or
some) of the $n$ inputs. This fitness value should be minimized; the optimal
solution has a fitness of $0$.

AND$_n$ is used to refer to the variant of this problem where the target
function is a conjunction of all $n$ input variables, while the target of
AND$_{n,m}$ is composed of an unknown subset of $m \leq n$ variables.
\end{gp-problem}

For example, when the complete truth table is used as the training set $T$, the
fitness of a tree containing only a single leaf $x_1$ for the AND$_n$ problem with $n=3$ is $3$,
while the fitness of the optimum is $0$ (the fitness function represents the
\emph{error} of the solution on the training set). In general, a conjunction of
$a$ distinct variables has a fitness of $2^{n-a}-1$ on the complete truth table.

The initial complexity analysis results for this problem consider the minimal
function set (i.e., $F=\{AND\}$) to simplify the analysis by forcing all
solutions considered by the GP algorithms to be conjunctions. This
simplification renders the fitness function unimodal, making the AND$_n$
problem somewhat similar to the \onemax benchmark problem for evolutionary
algorithms: the GP system simply has to collect all $n$ distinct variables
together in its solution, with the fitness of the current solution improving
with each distinct variable that is added. In this minimal setting,
initializing with larger trees makes the problem easier for the GP system, as
fewer variables would need to be inserted into the tree to complete the
conjunction. Thus, for complexity analysis results, the initial solution is
typically an empty tree.

Building upon these results, the impact of using richer function (e.g., by
introducing disjunctions~\cite{DoerrLO19} and negations) and terminal sets
(via the AND$_{n,m}$ problem) has been also been analyzed.

\subsubsection{Complete Truth Table, Minimal Terminal and Function Sets}

Mambrini~and~Oliveto showed that the \RLSGP and \RLSGPStar algorithms can
efficiently construct the optimal solution for the AND$_n$ problem when they 
use the complete truth table to evaluate solution fitness \cite{MambriniO16}.

\begin{theorem}[\cite{MambriniO16}]
The expected optimization time of \RLSGP and \RLSGPStar 
with $F=\{AND\}$ and $L := \{x_1, \ldots, x_n\}$ on the AND$_n$ problem using the
complete truth table as the training set is $\Theta(n \log n)$. The solution produced by
\RLSGPStar contains exactly $n$ terminals.
\end{theorem}

The proof applies a coupon collector argument, showing that with probability
$(n-i)/(3n)$ a new variable is added to the solution, and that no mutations
decreasing the number of distinct variables are ever accepted. As all internal
nodes are forced to be conjunctions, collecting all $n$ variables in the tree
produces an optimal solution.

The following theorem presents a fixed budget analysis of the \RLSGP and
\RLSGPStar algorithms, providing a relationship between the expected number
of distinct variables in the solution and the time the algorithms are allowed to run.
\begin{theorem}[\cite{LissovoiO18}]
Let $v(x)$ denote the number of distinct variables in solution $x$, and let
$x_b^*$ or $x_b$ be the solution produced by the \RLSGPStar or \RLSGP algorithms, respectively, with $F=\{AND\}$ and $L := \{x_1, \ldots, x_n\}$,
given a budget of $b$ iterations on the AND$_n$ problem using
the complete truth table as the training set when initialized with an empty tree. Then,
$$ E(v(x_b^*)) = n-n(1-1/(3n))^b, $$
$$ n- n(1-1/(3n))^b \leq E(v(x_b)) \leq n- n(1-2/(3n))^b. $$
\end{theorem}

The theorem is proven by following the techniques used to
analyze Randomized Local Search (RLS) on the \onemax problem in \cite{JansenZ14}.
The exact expectation is known for
\RLSGPStar, which never accepts solutions that do not improve fitness, and
hence can never have a substitution suboperation increase the number of
distinct variables in the solution.
The upper and
lower bounds on $E(v(x_b))$ for \RLSGP stem from trivial bounds on the
probability of a substitution suboperation of HVL-Prime increasing the number 
of distinct variables in the solution. 
We note that although the relationship $f(x) = 2^{n-v(x)}-1$ between the
solution fitness ($f(x)$) and the number of distinct variables it contains
($v(x)$) is known, it is not possible to apply linearity of expectation to
transform a bound on $E(v(x_b))$ into a bound on $E(f(x_b))$ (as could be done
for \onemax).

The runtime analysis results have been extended to cover the \OneOneGP algorithms, and show that
the expected number of terminals in the constructed solution is $\Theta(n)$.

\begin{theorem}[\cite{LissovoiO18}] \label{thm:AND-Space}
The expected optimization time of the \OneOneGP and the \OneOneGPStar 
with $F=\{AND\}$ and $L := \{x_1, \ldots, x_n\}$
on the AND$_n$ problem using the
complete truth table as the training set is $\Theta(n \log n)$. In expectation, the solution
produced by these algorithms contains $\Theta(n)$ terminals.
\end{theorem}

For the AND problem, there are many possible trees which encode the desired
behavior (because repeating a variable multiple times in the conjunction does not
negatively affect the behavior of the program) and it is therefore possible
that a ``correct'' program could contain many redundant leaf nodes.
The space complexity result in Theorem~\ref{thm:AND-Space} shows that the
considered GP systems construct a tree that in expectation contains just $O(n)$
leaf nodes. This is proven by showing that the number of
leaf nodes that contain variables present in the solution multiple times does not
grow fast enough to affect the asymptotic tree size bound in the $O(n \log n)$
iterations required to collect all $n$ variables with high probability.

\subsubsection{Incomplete Training Sets, Minimal Terminal and Function Sets}

In practice, it may not be possible to evaluate the exact fitness of a
candidate solution on all $2^n$ possible Boolean inputs when $n$ is large. If
this is the case, solution quality could instead be evaluated by executing the
program on a sampled subset of possible inputs (the ``training set''). Without
assuming any specific knowledge of the target function class, the training set
could be sampled uniformly at random.

When training sets of polynomial size sampled uniformly at random are used for
the AND$_n$ problem, a solution representing a conjunction of a logarithmic
number of distinct variables will with high probability be correct on all of
the inputs included in the training set.
This causes the optimization process to end prior to finding a
solution that is correct on all possible inputs \cite{MambriniO16}.
The following result holds both when the training set is sampled once
and for all at the beginning of the run (i.e., a static training set) and when
at each generation a new training set is sampled (i.e., a dynamic training set).

\begin{theorem}[\cite{MambriniO16,LissovoiO18}] \label{thm:gp-AND-polyset}
Let $s = \operatorname{poly}(n)$ be the size of a training set chosen from the truth
table uniformly at random with replacement. With $F=\{AND\}$ and $L := \{x_1, \ldots, x_n\}$, both \RLSGP and
\RLSGPStar will fit the training set on the AND$_n$ problem in expected time $O(\log s) =
O(\log n)$, and the solution will contain at most $O(\log n)$ variables.
\end{theorem}

This result is proven by observing that rows selected uniformly at random from
the truth table are unlikely to assign more than $Y = n/2 + \epsilon n$ input
variables to true, and hence can be satisfied by inserting any one of a linear
number of variables into the solution. After $\log_{n/Y}(2s)$ successful
insertions, the probability that some row of the $s$-row training set is still
not satisfied is at most $n/2$, and hence in expectation the process satisfies
all rows after $2k = O(\log n)$ distinct variables have been successfully
inserted into the tree.

Theorem~\ref{thm:gp-AND-polyset} also yields
a lower bound on the generalization error of the
solution: if it contains at most $O(\log n)$ variables, the probability that
its output is wrong on a truth table row sampled uniformly at random is
$2^{-O(\log n)} = n^{-O(1)}$, i.e., it requires in expectation a polynomial
number of samples taken uniformly at random from $C$ before a divergence from
the target function is discovered.

Theorem~\ref{thm:gp-AND-polyset} has been extended to cover
the \OneOneGP and \OneOneGPStar algorithms, using a multiplicative drift
theorem to provide a runtime bound on the expected time to fit a static 
polynomial-sized
training set \cite{LissovoiO18}.
Additionally, a similar bound holds if, instead of a static
training set, each iteration samples $s$ independent rows of the complete truth
table to compare the fitness of two solutions (using a dynamic training set).

\begin{theorem}[\cite{LissovoiO18}]
Let $s = n^{2c+\epsilon}$ rows from the complete truth table of the AND$_n$ problem
be sampled with replacement and uniformly at random in each iteration
(where $c > 0$ and $\epsilon > 0$ are any constants).
With $F=\{AND\}$ and $L := \{x_1, \ldots, x_n\}$, \RLSGP, \RLSGPStar, \OneOneGP, and \OneOneGPStar will construct a solution
with a generalization error of at most $n^{-c}$ in expected $O(\log n)$
iterations. In expected $O(\log^2 n)$ iterations, the nonstrictly elitist
algorithms will construct a solution with a sampled error of $0$.
\end{theorem}

Here, the training set size $s$ is chosen to be sufficiently large to ensure
that solutions with a generalization error greater than $n^{-c}$ are wrong
on at least one training set row with high probability, preventing the GP
system from terminating early with a bad solution, while the $O(\log^2 n)$
runtime bound stems from a random walk argument pessimistically considering the
probabilities of accepting solutions that increase or decrease the number of
distinct variables in the tree to be equal.

\subsubsection{More Expressive Function and Terminal Sets} \label{sec:gp-conj-expressive}
In practical applications of GP, it may not be known which functions or input
variables are useful for evolving the target function, and thus a generic GP 
system is usually
given access to a wide variety of functions and terminals. In the
setting of evolving conjunctions, this may be modeled by introducing input
variables not included in the target conjunction (the AND$_{n,m}$ problem), or
giving the GP systems access to additional Boolean operators (such as negation
or disjunction). The aim is to evaluate whether the systems are still able to 
evolve the target function efficiently.

The AND$_{n,m}$ problem is a variant of the AND problem in which the target
function is a conjunction of $m \leq n$ distinct variables from the terminal set
$L$. This is similar to the conjunction evolution problem considered by
Valiant~\cite{Valiant09} and has been analyzed for \RLSGP algorithms in
\cite{LissovoiO18}. The \RLSGP and \RLSGPStar algorithms
(the latter only when disallowing the HVL-Prime substitution suboperation)
are able to construct an optimal solution for
the AND$_{n,m}$ problem using the complete truth table in an expected
$O(n \log n)$ iterations, while the canonical
\RLSGPStar will with high probability fail to find the optimum. 

\begin{theorem}[\cite{LissovoiO18}]
The \RLSGP algorithm and the \RLSGPStar algorithm (without the HVL-Prime 
substitution suboperation) using $F=\{AND\}$ and $L := \{x_1, \ldots, x_n\}$ find the optimum for the AND$_{n,m}$ problem in expected $O(n \log n)$
iterations when using the complete truth table as the training set.

The \RLSGPStar algorithm (with the substitution suboperation) will with high
probability fail to find the optimum for the AND$_{n,m}$ problem when $m = cn$ for any
constant $0 < c < 1$ when using the complete truth table as the training set.
\end{theorem}

The analysis
relies on showing that, initially, inserting both variables that are present in
the target function (``correct'' variables) and those that are not
(``incorrect'' variables) is beneficial for the fitness value of the candidate
solution, while removing incorrect variables only becomes beneficial after all
correct variables are present in the current solution. With local search
mutation and the substitution suboperation of HVL-Prime, it is possible for
\RLSGPStar to accept a solution which substitutes the last copy of some
incorrect variable with another copy of a still-present incorrect variable in
the solution. If this occurs, \RLSGPStar will not be able to reach the global
optimum, because a single application of HVL-Prime could only remove a single
copy of an incorrect variable present multiple times in the current
solution, which would not provide a fitness improvement.

It is conjectured that a similar bound also holds for the runtime of the \OneOneGP
and \OneOneGPStar algorithms, which are able to introduce and remove duplicate
terminals in the solution using larger mutation operations.

A more realistic function set as used in practice should also include
additional Boolean operators, such as OR or NOT,
with the aim of giving the GP system the expressive
power necessary to represent any Boolean function. Mambrini and Oliveto have
shown that if the unary NOT operation is introduced (by extending the set
of literals with negated versions of each variable, avoiding the need to modify
the HVL-Prime mutation operator to deal with nonbinary functions), the \RLSGP
algorithms are no longer able to efficiently construct the optimal solution of
the AND problem using the complete truth table as the training set \cite{MambriniO16}. This 
result was extended by Lissovoi and Oliveto to cover the \OneOneGP algorithms
\cite{LissovoiO18}.

\begin{theorem}[\cite{MambriniO16,LissovoiO18}]
The \RLSGP, \RLSGPStar, \OneOneGP and \OneOneGPStar algorithms on the AND$_n$ problem with $L =
\{x_1, \ldots, x_n, \overline{x}_1, \ldots, \overline{x}_n\}$ and $F=\{AND\}$  do not construct
an optimal solution in polynomial time, with overwhelming probability, when using
the complete truth table as the training set.
\end{theorem}

The theorem follows from the
observation that a conjunction that contains both a variable $x_i$
and its negation $\overline{x}_i$ always evaluates to false, and hence has a
nearly optimal fitness value of $1$ (i.e., it is wrong on just one of $2^n$
possible inputs). Such a pair of literals was shown to be
present in the current solution with overwhelming probability once it contains
$n/2$ distinct literals. For the strictly elitist GP algorithms,
reaching the global optimum would then require a large simultaneous mutation
with an exponential waiting time, while the nonstrictly elitist GPs would
essentially need to perform a random walk in $2n$ dimensions and reach a particular
point while receiving little guidance from the fitness function.

Additionally, even if the GP systems could be prevented from accepting any
solution containing a contradiction (for instance, by weighting the all-true
variable assignment much higher than any other input), the \RLSGP and \OneOneGP
algorithms would still require exponential time to find the global optimum, as
nonoptimal solutions containing all $n$ variables (in either the positive or
the negated form) share the same fitness value ($2^n-2$, i.e., they are wrong
on the all-true input and the single assignment satisfying the solution but not 
the target function), and the closer the GP system is to having all $n$ positive
literals, the more likely it is to produce an offspring which replaces a
positive literal with a negative one.

On the other hand, if a training set of polynomial size is used as in practical
applications, the GP systems can still efficiently construct a solution which 
generalizes well (even if it is not optimal) on the AND$_n$ problem, even in
the presence of negations.

\begin{corollary}[\cite{LissovoiO18}] \label{cor:not-generalisation}
The \OneOneGP using $F = \{AND\}$ and $L = \{x_1, \ldots, x_n, \overline{x_1},
\ldots, \overline{x_n}\}$, is able to find a solution on the AND$_n$ problem
with a generalization
error of at most $n^{-c}$ for any constant $c > 0$ in polynomial time, when
comparing program quality using a sufficiently large training set of polynomial 
size chosen either uniformly at random from the complete truth table in each
iteration or during the first iteration.
\end{corollary}

Doerr et~al.\ \cite{DoerrLO19} have analyzed the behavior of the \RLSGP
algorithm using $F = \{AND, OR\}$ for the AND$_n$ problem. To allow the
analysis in this setting, a limit on the maximum solution size was imposed;
specifically, solutions containing more than $\ell \geq n$ leaf nodes were
rejected regardless of their fitness. However, there exist solutions with
$\ell$ leaf nodes which cannot be modified by HVL-Prime without detrimentally
affecting fitness, and hence \RLSGP requires an expected infinite number of
iterations to find an optimal solution. To address this issue,
the HVL-Prime deletion sub-operation was modified to select a node uniformly at
random and remove the subtree rooted at that node (replacing the node's parent
with the node's sibling). Allowing subtree deletions brings
the operator closer to the sort of large-scale modifications of candidate
solutions that are produced by the mutation operators of practical GP systems
\cite{PoliLMBook}. With the two modifications, \RLSGP is able to find the
global optimum in expected polynomial time with respect to the number of
variables and the limit on the tree size imposed if the complete truth table
is used.

\begin{theorem}[\cite{DoerrLO19}] \label{thm:gp-conj-andor-runtime}
The \RLSGP algorithm with $F=\{AND, OR\}$ and $L := \{x_1, \ldots, x_n\}$, a tree size limit $\ell \geq
(1+c) n$ leaf nodes for any $c \in \Theta(1)$, HVL-Prime with subtree deletion,
finds the optimum for the AND$_n$ problem in expected $O(\ell \, n \log^2 n)$
iterations when using the complete truth table as the training set.
\end{theorem}

This result was proven by showing that within $\Omega(\ell \, n \log^2 n)$
iterations, the current solution of the \RLSGP contains fewer than $\ell$ leaf
nodes, and thus progress can be made by inserting a conjunction with a useful
variable at the root of the offspring solution. A super-multiplicative drift
theorem was then applied to bound the expected runtime.
Experimental results suggest that a tree size limit is not required in this
setting, and that systems with larger tree size limits find the optimum in
fewer iterations than those with tree size limits close to $n$ \cite{DoerrLO19}.

When using incomplete training sets to evaluate solution quality, it was shown
that with probability $1-O(\log^2(n)/n)$, \RLSGP avoids inserting any
disjunctions before finding a solution which satisfies its termination
condition and with high probability reaches the desired generalization ability.

\begin{theorem}[\cite{DoerrLO19}] \label{thm:incomplete-single-run}
For any constant $c > 0$, consider an instance of the
\RLSGP algorithm with $F = \{AND, OR\}$, $L = \{x_1, \ldots, x_n\}$, a tree size limit $\ell \geq n$,
using a training set of $s = n^c \lg^2 n$ rows sampled uniformly
at random from the complete truth table in each iteration to evaluate solution
quality, and terminating when the sampled error of the solution is at most
$c' \lg n$, where $c'$ is an appropriately large constant. On the AND$_n$
problem, the algorithm will, with probability at least $1 - O(\log^2(n)/n)$,
terminate within $O(\log n)$ iterations, and return a solution with a generalization
error of at most $n^{-c}$.
\end{theorem}

Notably, the theorem does not require an \emph{upper} limit on the size of the
tree; $\ell \geq n$ simply ensures that the target function is representable
within the tree size limit. The proof shows that a solution with the desired
generalization error is found once $O(\log n)$ insertions occur, and thus the
\RLSGP with high probability does not exceed any reasonable tree size limit in
this setting. Experimental results additionally show that solutions with fewer
undesired disjunctions could be constructed by terminating the GP system once
it achieves a logarithmic error on the training set rather than waiting for an
error of 0 to be observed \cite{DoerrLO19}.

\subsubsection{Optimal Training Sets}
While the target conjunctions are unlikely to be evolved exactly (with a
generalization error of 0) when using a polynomial training set chosen
uniformly at random, there do exist small training sets of $O(n)$ rows which
allow the \RLSGP and \OneOneGP algorithms to find exact solutions efficiently.
In general, identifying such training sets may be nontrivial.

\begin{theorem}[\cite{LissovoiO18}] \label{thm:gp-and-opt-training-sets}
Let $M$ be an $n$-row training set, where row $i$ sets $x_i$ to false and all
$x_j$ (where $j \neq i$) to true, and let $M'$ be a $2n+1$-row training set
containing all the rows of $M$ and $n+1$ copies of the row setting all inputs
to true.
The \RLSGP and \OneOneGP algorithms with $F=\{AND\}$ using the
training sets $M$ and $M'$, respectively are able to find the exact solution of AND$_n$ and
AND$_{n,m}$ with $F = \{AND\}$, $L = \{x_1, \ldots, x_n\}$ (or AND$_n$
with $F = \{AND\}$ and $L = \{x_1, \ldots, x_n, 
\overline{x}_1, \ldots, \overline{x}_n\}$) in expected $O(n
\log n)$ fitness evaluations (or $O(n^2 \log n)$ training set row evaluations).
\end{theorem}

For $L = \{x_1, \ldots, x_n, 
\overline{x}_1, \ldots, \overline{x}_n\}$, a variant of the
\OneOneGP which maintains and randomly selects from a population of $\mu$
individuals subject to a diversity mechanism prohibiting multiple solutions
with identical outputs on the training set
was proven to find an optimal solution in $O(\mu n \log n)$ iterations
on an $n+1$-row training set (consisting of
all the inputs in $M$ and an input where all the $n$ variables are set to true)
\cite{LissovoiO18}. Effectively, this uses the explicit diversity mechanism to 
avoid including multiple copies of the all-true row in the training set as in Theorem~\ref{thm:gp-and-opt-training-sets}.

\subsection{Evolving Parity}
The XOR problem asks the GP system to evolve an exclusive disjunction of all
$n$ input variables. Unlike conjunctions, exclusive disjunctions are known to
not be evolvable in the PAC learning framework~\cite{Valiant09}.

\begin{gp-problem}[XOR]
Let $L \subseteq \{x_1, \ldots, x_n\}$ be the set of available terminals, and
$F$ be the set of available functions.

The fitness of a tree $X$ using a training set $T$ selected from the rows of the
complete truth table $C$ is the number of training set rows on which the value
produced by evaluating the Boolean expression represented by the tree
differs from the output of the exclusive disjunction of all $n$ inputs.
\end{gp-problem}

When $F={XOR}$ and the complete truth table is used as the training set,
the fitness of any nonoptimal solution is $2^{n-1}$, while the fitness of the 
optimal solution is $0$. Thus, using the complete truth table as the training 
set on XOR is similar to the Needle
benchmark problem; Langdon~and~Poli noted that ``the fitness landscape is
like a needle-in-a-haystack, so any adaptive search approach will have
difficulties''~\cite{LangdonPBook}.

Predictably, the \RLSGP and \OneOneGP algorithms are not able to optimize XOR
efficiently. Strictly elitist variants will only move from their initial
solution if the optimum is constructed as a mutation of that solution, which
occurs in expected infinite time for \RLSGPStar (as the optimum is not
reachable by a single HVL-Prime mutation from many possible points), and in
expected exponential time for the \OneOneGPStar (which essentially needs to
construct the complete function in one mutation; if initialized with an empty
tree, this mutation needs to perform at least $n$ HVL-Prime insertion
suboperations).
When the complete truth table set is used as the training set, the expected
optimization time for \RLSGP is exponential in the problem size,
because the algorithm accepts any and
all mutations, while reaching the optimal solution requires all $n$ variables
to appear an odd number of times in the solution \cite{MambriniO16}.

\begin{theorem}[Theorem 4, \cite{MambriniO16}] \label{thm:gp-xor}
\RLSGP using $F=\{XOR\}$, $L=\{x_1, \ldots, x_n\}$, and using the complete
truth table as the training set to evolve XOR$_n$ requires more than
$2^{\Omega\left(n / {\log n} \right)}$ iterations with probability $p > 1 -
2^{-\Omega\left(n / {\log n}\right)}$ to reach the optimum.
\end{theorem}

The theorem is proven by an application of the simplified negative drift
theorem, showing that when the number of variables that appear in the current
solution an odd number of times is large, there is a strong negative drift
towards reducing this number, and the optimum requires all $n$ distinct
variables to appear an odd number of times in the solution. The negative drift
stems primarily from the HVL-Prime insertion operator: if a large number of
variables are represented an odd number of times, it is more likely to insert
one of these variables when choosing a terminal uniformly at random.

Also when sampling solution fitness using a polynomial number of rows of the
complete truth table, the outcome is underwhelming: if only a
logarithmically small number of training set rows are sampled in each
iteration, the algorithm will terminate in expected polynomial time with a
nonoptimal solution that fits the sampled training set, while using training
sets of superlogarithmic size will lead to superpolynomial optimization time.
Thus, in any polynomial amount of time, the expected generalization ability of
the GP systems considered on XOR is $1/2$, i.e., they require in expectation a
constant number of samples taken uniformly at random from $C$ before a
divergence from the target function is discovered.

There is also a straightforward extension of Theorem~\ref{thm:gp-xor} to
dynamic training sets of polynomial size, as such sampling provides no
consistent indication of fitness.

\begin{corollary}
The \RLSGP and \OneOneGP algorithms sampling $s \in \omega(\log n)$ rows of the
complete truth table in each iteration on XOR$_n$ with $F=\{XOR\}$ and $L=\{x_1, \ldots, x_n\}$
with high probability do not reach the optimum in polynomial time.
\end{corollary}
\begin{proof}
The \RLSGP and \OneOneGP algorithms will accept \emph{any} nonoptimal offspring of a
nonoptimal parent with probability at least $1/2$, as both the offspring and
the parent are wrong on $2^{n-1}$ inputs, and there are exactly as many rows
on which the offspring is correct while the parent is wrong as the converse,
and the offspring is accepted in cases of tied fitness.

With $s \in \omega(\log n)$ rows sampled uniformly at random in each iteration, the probability
that a nonoptimal solution is correct on all sampled rows is $2^{-\omega(\log
n)} = n^{-\omega(1)}$, and, by a straightforward union bound, the GP algorithms
do not terminate within polynomial time unless the optimal solution is
found.

With the exception of any iterations in which the offspring individual is
rejected, the algorithms behave identically to the \RLSGP and \OneOneGP algorithms using the
complete truth table to evaluate solution fitness (i.e., accepting offspring
regardless of the effects of mutation), and thus cannot achieve better 
performance than these algorithms in terms of the number of iterations 
performed.

Theorem~\ref{thm:gp-xor} provides a runtime bound for \RLSGP only. A
similar result for the \OneOneGP can be obtained by observing that the \OneOneGP
performs in expectation two HVL-Prime suboperations in each iteration, and
hence, even if the algorithm terminated immediately upon constructing the
optimal solution (even if this occurred in the middle of a mutation), it would
in expectation be only a constant factor faster than \RLSGP in terms of the
number of iterations required to find the optimum.
\hfill~\qed
\end{proof}

\subsection{Outlook}

In this section, the available computational complexity results regarding the
evolution of proper functions with input/output behavior have been overviewed.
Simple GP systems equipped with the AND (or AND and OR) functions and positive
literals (or possibly both positive and negative literals) can evolve
conjunctions of arbitrary size with high probability if appropriate limits on
maximum tree size are put in place. Important open problems are providing
performance statements of the algorithms without tree size limits, and analyses
of GP systems equipped with comprehensive function sets $F$, i.e., those that
allow the expression of any Boolean function.

\section{Other GP Algorithms} \label{sec:gp-other-algs}
The previous sections have covered the available theoretical results for
standard tree-based GP systems, which constitute the majority of theoretical
complexity analysis results for GP. Several other GP paradigms have been
proposed in the literature which use different representations for candidate
solutions, e.g., Cartesian GP \cite{Miller11book}, Linear GP
\cite{BrameierBanzhaf07book}, PushGP \cite{SpectorR02}, and Geometric Semantic
GP (GSGP) \cite{MoraglioKJ12}. Amongst these, the only class for which
computational complexity analyses are available is GSGP. In this section, we
present the available results concerning this different approach to GP system
design which aims to evolve programs semantically rather than syntactically.

\subsection{Geometric Semantic Genetic Programming}

Standard tree-based GP evolves programs by applying mutation and crossover to
their syntax. Programs that are considerably different syntactically may
produce identical output, while introducing minimal syntactic mutations may
completely change the output of a program. Moraglio~et~al.~\cite{MoraglioKJ12} introduced
Geometric Semantic GP (GSGP) with the aim of focusing GP search on program
behavior. In particular, GSGP mutation and crossover operators modify programs
in a way that allows the GP system to search through the semantic neighborhood
(which consists of programs with similar behavior) rather than their syntactic
neighborhood (which consists of programs with similar syntax).

GSGP generally uses a natural program representation for the domain at hand
(e.g., it represents programs using Boolean expressions when a Boolean expression
is to be evolved), and uses specialized semantic mutation and crossover
operators to produce offspring programs with \emph{behavior} similar to that of
their parents. These operators generally reproduce the parent programs in their
entirety, adding to them to modify their behavior in a limited fashion. For
example, the GSGP mutation operator could produce an offspring which contains
an exact copy of its parent and a random element which overrides some portions
of the parent's behavior, while the GSGP crossover operator could construct an
offspring containing exact copies of both parents and a random element which
switches between the two behaviors depending on the inputs. As both operators
increase the size of the programs by adding additional syntax
to the parent programs to encode the chosen random components (and the crossover includes 
exact copies of \emph{both} parents), the programs produced by these operators
need to be simplified in order for the algorithms to remain tractable. 
For some domains, such as Boolean functions, quick function-preserving
simplifiers exist, while computer algebra systems and static analysis can be
used to simplify more complex expressions and programs \cite{MoraglioKJ12}.

Semantic geometric crossover and mutation operators have been designed for many
problem domains, including regression problems \cite{MoraglioM13}, learning
classification trees \cite{MambriniMM13}, and Boolean functions
\cite{MoraglioMM13}.
Initial experimental results suggest that GSGP consistently finds solutions
that fit the training sets used for a wide array of simple Boolean benchmark
functions, regression problems for polynomials of degree up to 10, and various
classification problems, outperforming standard tree-based GP with the same
evaluation budget \cite{MoraglioKJ12}.  Theoretical guarantees
have been derived regarding the number of generations it takes GSGP to
construct a solution fitting the training set, or achieving an $\epsilon$-small
training set error in the case of regression problems 
\cite{MoraglioM13,MambriniMM13,MoraglioMM13}. In this section, we
explore the available theoretical results focusing on applying
geometric semantic search to evolving Boolean functions.

In the case of Boolean functions, the program semantics can be represented by
a $2^n$-row output vector, corresponding to the program output on all rows of
the complete $n$-variable truth table. In this setting, the semantic crossover
operator SGXB, acts on two parents $T_1$ and $T_2$, and produces an offspring
solution $(T_1 \wedge T_\mathrm{R}) \vee (T_2 \wedge \overline{T_\mathrm{R}})$, where $T_\mathrm{R}$ is a
randomly generated Boolean function. This offspring outputs the solution
produced by $T_1$ if $T_\mathrm{R}$ evaluates to true, and the solution produced by
$T_2$ if $T_\mathrm{R}$ evaluates to false, effectively performing crossover on the
$2^n$-row output vectors of the two parent solutions. The semantic mutation
operator SGMB, acting on a single parent $T_1$, produces the offspring $T_1
\vee M$ with probability 0.5, and $T \wedge \overline{M}$ with probability 0.5,
where $M$ is a random minterm (a conjunction where each variable appears in 
either positive or negated form) of all input variables. This effectively copies the
output vector of $T_1$, setting the rows on which $M$ evaluates to true to
either true or false.

These operators allow GSGP to always observe a cone fitness landscape on any
Boolean function, i.e., the mutation operator is always able to improve the
behavior of the parent program. This allows mutation-only GSGP to hill-climb
its way up to the optimal program for any function in this domain.
However, since the output vector contains $2^n$ rows,
hill-climbing by applying SGMB, which only affects one row per iteration, would
take $O(2^n \log(2^n)) = O(n2^n)$ iterations (by the coupon collector argument,
or similarly to RLS on a $2^n$-bit \onemax function).

For GSGP on any Boolean function, a polynomially sized training set can be
viewed as a \onemax problem on a $2^n$-bit string where only a polynomial number
of bits are nonneutral (i.e., contribute to the solution's fitness).
In that setting, the runtime can be improved by
allowing mutations to flip more than one bit of the output vector per iteration
(e.g., such that in expectation one nonneutral bit is affected per iteration).
This setting was explored in \cite{MoraglioMM13}, with various approaches to the design
of mutation operators, establishing a hierarchy of operator expressiveness
(based on how much of the search space they enable the GP system to explore), and
considering the probability of fitting a training set of polynomial size.
The following mutation operators, differing in how the random minterm $M$ used
to modify program behavior is constructed, were analyzed:
\begin{itemize}
\item Fixed Block Mutation (FBM), which picks the $v \leq n$ variables to
use as the base for $M$ \emph{once} during the run,
\item Fixed Alternative Block Mutation (FABM), which partitions the variables
into $v$ sets, and forms $M$ by picking a variable from each set uniformly at random in each iteration,
\item Varying Block Mutation (VBM), which
in each iteration chooses $v \leq n$ variables uniformly at random to form the base for $M$.
\end{itemize}
For all three operators, $v$ is a fixed parameter.
The results show that while VBM is more expressive than FABM, which in turn is
more expressive than FBM, there nevertheless exist training sets which GSGP
using VBM cannot fit in any amount of time. Conversely,
the less expressive
FBM operator can with high probability fit a training set of polynomial size sampled
uniformly at random from the complete truth table of any Boolean function
\cite{MoraglioMM13}.

\begin{theorem}[\cite{MoraglioMM13}] \label{thm:GSGP-VBM}
Let a training set $T$ consist of $n^c$ rows, with $c$ a positive constant,
the rows being
sampled uniformly at random from the complete truth table of any Boolean 
function. Then GSGP using the FBM 
operator with $v = (2c + \epsilon) \log_2(n)$ (for any $\epsilon > 0$),
is able to fit $T$ with probability at least $1-\tfrac{1}{2}n^{-\epsilon}$.
Conditioning on this, a function that fits the training set is found in an
expected $O(n^{2c}\log n)$ iterations.
\end{theorem}

This result is proven by observing that FBM's initial choice of $v$ variables
(to use as the basis for the minterms) partitions the $2^n$ row output vector
of $P$ into $2^v$ blocks of equal size, each corresponding to a particular minterm of the $v$
variables. Choosing $v > 2c \log_2 n$ partitions the output vector into more than
$2^{2c \log_2 n} = n^{2c}$ blocks, ensuring that with high probability all
$n^c$ training set rows (chosen uniformly at random from the
complete truth table) are in different blocks, and thus the training set can
be satisfied by collecting the exact minterms corresponding to the blocks which
contain the training set rows. When this condition holds, the expected runtime
is obtained by a coupon collector argument.

Of course, if FBM chooses the $v$ variables poorly with respect to the training
set $T$ (meaning that at least two training set rows demanding different output are
contained in the same block), GSGP will not be able to fit it. More
expressive operators such as FABM or VBM can minimize this probability at the cost of a mild runtime penalty by
allowing the mutation operator to be more flexible when choosing which variables
to use as the basis for the minterm (e.g., increasing the runtime by a factor of $n/v$, but improving the success probability from $p$ to $1 - (1-p)^{n/v}$, where $v$ is the number of classes in the partition created by FABM).

There are also modifications of the GSGP mutation operators that are able to
cover the entire search space of programs, eliminating the possibility of
failure. There exist classes of Boolean functions for which such operators are
effective, allowing GSGP to fit any training set in expected polynomial
time (with no failure probability, unlike Theorem~\ref{thm:GSGP-VBM}), as shown in the following theorem.
\begin{theorem}[\cite{MoraglioMM13}]
Let $\phi$ be a formula in disjunctive normal form with $\alpha = \operatorname{poly}(n)$ conjunctions, every
conjunction containing at most $\beta = O(1)$ variables. Then
GSGP with Multiple Size Block Mutation (MSBM) can fit any training set for
$\phi$ in expected $O(\alpha n^{\beta+1}2^\beta)$ iterations,
i.e., polynomial time.
\end{theorem}

The MSBM mutation operator is a modification of the VBM variant of the SGMB operator. It samples an
integer $v$ between $0$ and $n$, selects $v$ variables from the set of $n$
input variables, and then generates uniformly at random an incomplete minterm
$M$ of these variables. This modified mutation operator essentially allows each
clause of the target function to be ``fixed'' in the current solution in an
expected polynomial number of iterations.

\subsection{Outlook}

GSGP systems have been proven to efficiently construct solutions which fit
training sets of polynomial size for several function domains. In this section,
we have covered the available results for the evolution of Boolean functions,
although similar results are available for the other domains, such as learning
classification trees and regression problems \cite{MambriniMM13,MoraglioM13}.

At present though, there are no theoretical analyses of how the functions
produced by GSGP generalize to unseen inputs. Experimental results concerning
the generalization performance of GSGP systems yielded mixed
conclusions~\cite{GoncalvesSF15,PawlakK17,OrzechowskiCM18}.

\section{Conclusion} \label{sec:gp-conclusion}

We have presented an overview of the available results on the computational
complexity analysis of GP algorithms. The results follow the blueprint
suggested by Poli~et~al., starting with the analysis of simple GP systems
based on mutation and stochastic-hill climbing on simple problems \cite{PoliVLM10}. The
complexity of the problems has gradually increased, from analyses focusing
on the main characteristic difficulties of GP (i.e., variable solution length,
and solution quality evaluations) to more recent results
considering the evolution of functions with true input/output behavior and
using realistically constrained fitness functions. The approach of gradually
expanding the complexity of the systems analyzed was also endorsed by
Goldberg~and~O'Reilly, who stated that ``the methodology of using deliberately
designed problems, isolating specific properties, and pursuing, in detail,
their relationships in simple GP is more than sound; it is the only practical
means of systematically extending GP understanding and design'' \cite{GoldbergO98}.

The GP systems considered in theoretical analyses have remained relatively
simple: the use of HVL-Prime mutation and limited, if
any, populations with no crossover is a common setting. In many cases, an
analysis that provides positive runtime results is only made tractable because
``the fitness structure of the model problems is simple, and the algorithms use
only a simple hierarchical variable length mutation operator''
\cite{DurrettNO11}. In particular, variable-length representations
often complicate the analysis of GP systems, and require ``rather deep insights
into the optimization process and the growth of the GP-trees''
\cite{DoerrKLL17}.

The chapter has highlighted three different streams that have been followed for
building the theoretical foundations of genetic programming.
The first one is
the design and analysis of benchmark functions with variable length
representation for the analysis of tree structure growth. Three classes of such
problems have been considered in the literature: \Order, \Majority, and \Sorting.
While producing rigorous proofs is not easy, surprisingly simple hillclimbing
GP systems optimize these problems efficiently without bloat seriously
hindering their performance. Only recently has the 2/3-\textsc{SuperMajority}
benchmark function been introduced as a benchmark problem where bloat provably
is a major concern. Nevertheless, simple bloat control mechanisms address the
issue effectively. As a result, there is a need for better benchmark functions
to shed more light on how bloat affects evolution via GP.

The second line of
research has addressed the evolution of toy programs of fixed size. The aim is
to analyze GP behavior when tree structure is constrained (e.g., with tree size
limits in place, as in the MAX problem) and solution quality estimation using
training sets of limited size (i.e., how large training sets have to be for
efficient evolution, e.g., the \IdentificationP problem). Only very preliminary
results are available addressing these questions: tight bounds are unavailable
for the MAX problem even for simple hillclimbing (1+1) GP algorithms and
\IdentificationP problem results are available only for very simple linear
functions.

The third line of research concerns the evolution of proper functions with
inputs and outputs. Up to today, only conjunctions and parity Boolean functions
have been considered for (1+1)~GP systems using limited function sets (i.e.,
that do not have sufficient expressive power to express all Boolean functions).
Nevertheless, such GP systems can provably evolve conjunctions of arbitrary
sizes with proper tree size limits in place.

For GP systems utilizing geometric semantic mutation and crossover operators,
analyses of the time required to produce a solution fitting the training set
are available for wider classes of functions, and frequently do not require
insight into the structure of the function considered. However, a rigorous
understanding of how well the GSGP solutions generalize -- how well they
perform on inputs not included in the training set -- remains a challenge.

While the results presented represent the first steps in the rigorous analysis of
the behavior of GP systems, bridging the gap to the GP systems used in practice
requires analyzing more complex GP algorithms on more realistic problems. Thus,
extending the results presented to broader classes of problems (for instance,
those allowing more flexibility in program behavior), to other problem classes
on which GP experimentally performs well (such as symbolic regression), and to
more realistic GP algorithms (introducing populations and crossover)
constitute the main directions for further research.

{\vspace{1ex}
\small
\noindent
\textbf{Acknowledgements} Financial support by the Engineering and
Physical Sciences Research Council (EPSRC Grant No. EP/M004252/1) is gratefully
acknowledged.}


\begin{thebibliography}{10}
\providecommand{\url}[1]{{#1}}
\providecommand{\urlprefix}{URL }
\expandafter\ifx\csname urlstyle\endcsname\relax
  \providecommand{\doi}[1]{DOI~\discretionary{}{}{}#1}\else
  \providecommand{\doi}{DOI~\discretionary{}{}{}\begingroup
  \urlstyle{rm}\Url}\fi

\bibitem{BrameierBanzhaf07book}
Brameier, M., Banzhaf, W.: Linear Genetic Programming.
\newblock Genetic and Evolutionary Computation. Springer (2007)

\bibitem{CorusDEL17}
Corus, D., Dang, D.C., Eremeev, A.V., Lehre, P.K.: Level-based analysis of
  genetic algorithms and other search processes.
\newblock IEEE Transactions on Evolutionary Computation \textbf{22}(5),
  707--719 (2017)

\bibitem{CorusO17}
Corus, D., Oliveto, P.S.: Standard steady state genetic algorithms can
  hillclimb faster than mutation-only evolutionary algorithms.
\newblock {IEEE} Transactions on Evolutionary Computation \textbf{22}(5),
  720--732 (2018)

\bibitem{CorusOY19}
Corus, D., Oliveto, P.S., Yazdani, D.: On inversely proportional hypermutations
  with mutation potential.
\newblock In: Proceedings of the Genetic and Evolutionary Computation
  Conference ({GECCO} 2019) (to appear) (2019).
\newblock {arXiv:1903.11674}

\bibitem{CurryLH07}
Curry, R., Lichodzijewski, P., Heywood, M.I.: Scaling genetic programming to
  large datasets using hierarchical dynamic subset selection.
\newblock {IEEE} Trans. Systems, Man, and Cybernetics, Part {B} \textbf{37}(4),
  1065--1073 (2007)

\bibitem{DoerrKLL17}
Doerr, B., K{\"{o}}tzing, T., Lagodzinski, J.A.G., Lengler, J.: Bounding bloat
  in genetic programming.
\newblock In: Proceedings of the Genetic and Evolutionary Computation
  Conference ({GECCO} 2017), pp. 921--928 (2017)

\bibitem{DoerrLO19}
Doerr, B., Lissovoi, A., Oliveto, P.S.: Evolving boolean functions with
  conjunctions and disjunctions via genetic programming.
\newblock In: Proceedings of the Genetic and Evolutionary Computation
  Conference ({GECCO} 2019) (to appear) (2019).
\newblock {arXiv:1903.11936}

\bibitem{DrosteJW02}
Droste, S., Jansen, T., Wegener, I.: On the analysis of the {(1+1)}
  evolutionary algorithm.
\newblock Theoretical Computer Science \textbf{276}(1-2), 51--81 (2002)

\bibitem{DurrettNO11}
Durrett, G., Neumann, F., O'Reilly, U.: Computational complexity analysis of
  simple genetic programming on two problems modeling isolated program
  semantics.
\newblock In: Proceedings of the 11th International Workshop on Foundations of
  Genetic Algorithms (FOGA 2011), pp. 69--80 (2011)

\bibitem{Feldman12}
Feldman, V.: A complete characterization of statistical query learning with
  applications to evolvability.
\newblock Journal of Computer and System Sciences \textbf{78}(5), 1444--1459
  (2012)

\bibitem{GathercoleR94}
Gathercole, C., Ross, P.: Dynamic training subset selection for supervised
  learning in genetic programming.
\newblock In: Proceedings of the 3rd International Conference on Parallel
  Problem Solving from Nature ({PPSN} 1994), pp. 312--321 (1994)

\bibitem{GathercoleR96}
Gathercole, C., Ross, P.: An adverse interaction between crossover and
  restricted tree depth in genetic programming.
\newblock In: Proceedings of the 1st Annual Conference on Genetic Programming,
  pp. 291--296. MIT Press, Cambridge, MA, USA (1996)

\bibitem{GoldbergO98}
Goldberg, D.E., O'Reilly, U.: Where does the good stuff go, and why? how
  contextual semantics influences program structure in simple genetic
  programming.
\newblock In: Proceedings of Genetic Programming, First European Workshop
  (EuroGP 1998), pp. 16--36 (1998)

\bibitem{GoncalvesSF15}
Gon{\c{c}}alves, I., Silva, S., Fonseca, C.M.: On the generalization ability of
  geometric semantic genetic programming.
\newblock In: Proceedings of Genetic Programming - 18th European Conference
  (EuroGP 2015), pp. 41--52 (2015)

\bibitem{Jansen13Book}
Jansen, T.: Analyzing Evolutionary Algorithms - The Computer Science
  Perspective.
\newblock Natural Computing Series. Springer (2013)

\bibitem{JansenZ14}
Jansen, T., Zarges, C.: Performance analysis of randomised search heuristics
  operating with a fixed budget.
\newblock Theoretical Computer Science \textbf{545}, 39--58 (2014)

\bibitem{KotzingLLM18}
K{\"{o}}tzing, T., Lagodzinski, J.A.G., Lengler, J., Melnichenko, A.:
  Destructiveness of lexicographic parsimony pressure and alleviation by a
  concatenation crossover in genetic programming.
\newblock In: Proceedings of the 15th International Conference on Parallel
  Problem Solving from Nature ({PPSN} 2018), Part {II}, pp. 42--54 (2018)

\bibitem{KotzingNS11}
K{\"{o}}tzing, T., Neumann, F., Sp{\"{o}}hel, R.: {PAC} learning and genetic
  programming.
\newblock In: Proceedings of the Genetic and Evolutionary Computation
  Conference ({GECCO} 2011), pp. 2091--2096 (2011)

\bibitem{KotzingSNO14}
K{\"{o}}tzing, T., Sutton, A.M., Neumann, F., O'Reilly, U.: The max problem
  revisited: The importance of mutation in genetic programming.
\newblock Theoretical Computer Science \textbf{545}, 94--107 (2014)

\bibitem{KozaBook92}
Koza, J.R.: Genetic programming - on the programming of computers by means of
  natural selection.
\newblock Complex adaptive systems. {MIT} Press (1992)

\bibitem{Koza10}
Koza, J.R.: Human-competitive results produced by genetic programming.
\newblock Genetic Programming and Evolvable Machines \textbf{11}(3-4), 251--284
  (2010)

\bibitem{KozaAJ08}
Koza, J.R., Al{-}Sakran, S.H., Jones, L.W.: Automated \emph{ab initio}
  synthesis of complete designs of four patented optical lens systems by means
  of genetic programming.
\newblock Artificial Intelligence for Engineering Design, Analysis and
  Manufacturing \textbf{22}(3), 249--273 (2008)

\bibitem{LangdonPBook}
Langdon, W.B., Poli, R.: Foundations of genetic programming.
\newblock Springer (2002)

\bibitem{Lipson08}
Lipson, H.: Evolutionary synthesis of kinematic mechanisms.
\newblock Artificial Intelligence for Engineering Design, Analysis and
  Manufacturing \textbf{22}(3), 195--205 (2008)

\bibitem{LissovoiO18}
Lissovoi, A., Oliveto, P.S.: On the time and space complexity of genetic
  programming for evolving boolean conjunctions.
\newblock In: Proceedings of the Thirty-Second {AAAI} Conference on Artificial
  Intelligence, pp. 1363--1370. {AAAI} Press (2018)

\bibitem{LohnHL08}
Lohn, J.D., Hornby, G., Linden, D.S.: Human-competitive evolved antennas.
\newblock Artificial Intelligence for Engineering Design, Analysis and
  Manufacturing \textbf{22}(3), 235--247 (2008)

\bibitem{LukeP02}
Luke, S., Panait, L.: Lexicographic parsimony pressure.
\newblock In: Proceedings of the Genetic and Evolutionary Computation
  Conference ({GECCO} 2002), pp. 829--836 (2002)

\bibitem{MambriniMM13}
Mambrini, A., Manzoni, L., Moraglio, A.: Theory-laden design of mutation-based
  geometric semantic genetic programming for learning classification trees.
\newblock In: Proceedings of the Congress on Evolutionary Computation ({CEC}
  2013), pp. 416--423 (2013)

\bibitem{MambriniO16}
Mambrini, A., Oliveto, P.S.: On the analysis of simple genetic programming for
  evolving boolean functions.
\newblock In: Proceedings of Genetic Programming - 19th European Conference
  (EuroGP 2016), pp. 99--114 (2016)

\bibitem{McDermottO15}
McDermott, J., O'Reilly, U.: Genetic programming.
\newblock In: Springer Handbook of Computational Intelligence, pp. 845--869.
  Springer (2015)

\bibitem{Miller11book}
Miller, J.F. (ed.): Cartesian Genetic Programming.
\newblock Natural Computing Series. Springer (2011)

\bibitem{MitzenmacherUpfalBook}
Mitzenmacher, M., Upfal, E.: Probability and computing - randomized algorithms
  and probabilistic analysis.
\newblock Cambridge University Press (2005)

\bibitem{MoraglioKJ12}
Moraglio, A., Krawiec, K., Johnson, C.G.: Geometric semantic genetic
  programming.
\newblock In: Proceedings of the 12th International Conference on Parallel
  Problem Solving from Nature ({PPSN} 2012), pp. 21--31 (2012)

\bibitem{MoraglioM13}
Moraglio, A., Mambrini, A.: Runtime analysis of mutation-based geometric
  semantic genetic programming for basis functions regression.
\newblock In: Proceedings of the Genetic and Evolutionary Computation
  Conference ({GECCO} 2013), pp. 989--996 (2013)

\bibitem{MoraglioMM13}
Moraglio, A., Mambrini, A., Manzoni, L.: Runtime analysis of mutation-based
  geometric semantic genetic programming on boolean functions.
\newblock In: Proceedings of the 12th International Workshop on Foundations of
  Genetic Algorithms (FOGA 2013), pp. 119--132 (2013)

\bibitem{MyersW03}
Myers, A.N., Wilf, H.S.: Some new aspects of the coupon collector's problem.
\newblock {SIAM} Journal on Discrete Mathematics \textbf{17}(1), 1--17 (2003)

\bibitem{Neumann12}
Neumann, F.: Computational complexity analysis of multi-objective genetic
  programming.
\newblock In: Proceedings of the Genetic and Evolutionary Computation
  Conference ({GECCO} 2012), pp. 799--806 (2012)

\bibitem{NguyenUW13}
Nguyen, A., Urli, T., Wagner, M.: Single- and multi-objective genetic
  programming: new bounds for weighted order and majority.
\newblock In: Proceedings of the 12th International Workshop on Foundations of
  Genetic Algorithms (FOGA 2013), pp. 161--172 (2013)

\bibitem{OlivetoW14}
Oliveto, P.S., Witt, C.: On the runtime analysis of the simple genetic
  algorithm.
\newblock Theoretical Computer Science \textbf{545}, 2--19 (2014)

\bibitem{OlivetoW15}
Oliveto, P.S., Witt, C.: Improved time complexity analysis of the simple
  genetic algorithm.
\newblock Theoretical Computer Science \textbf{605}, 21--41 (2015)

\bibitem{OlivetoX11}
Oliveto, P.S., Yao, X.: Runtime analysis of evolutionary algorithms for
  discrete optimization.
\newblock In: A.~Auger, B.~Doerr (eds.) Theory of Randomized Search Heuristics:
  Foundations and Recent Developments, chap.~2, pp. 21--52. World Scientific
  (2011)

\bibitem{OReillyO94}
O'Reilly, U., Oppacher, F.: Program search with a hierarchical variable lenght
  representation: Genetic programming, simulated annealing and hill climbing.
\newblock In: Proceedings of the 3rd International Conference on Parallel
  Problem Solving from Nature ({PPSN} 1994), pp. 397--406 (1994)

\bibitem{OReillyPhD}
O'Reilly, U.M.: An analysis of genetic programming.
\newblock Ph.D. thesis, Carleton University, Ottawa-Carleton Institute for
  Computer Science, Ottawa, Ontario, Canada (1995)

\bibitem{OReillyO96}
O'Reilly, U.M., Oppacher, F.: A comparative analysis of genetic programming.
\newblock In: P.J. Angeline, K.E. {Kinnear, Jr.} (eds.) Advances in Genetic
  Programming 2, chap.~2, pp. 23--44. MIT Press, Cambridge, MA, USA (1996)

\bibitem{OrzechowskiCM18}
Orzechowski, P., Cava, W.L., Moore, J.H.: Where are we now?: a large benchmark
  study of recent symbolic regression methods.
\newblock In: Proceedings of the Genetic and Evolutionary Computation
  Conference ({GECCO} 2018), pp. 1183--1190 (2018)

\bibitem{PawlakK17}
Pawlak, T.P., Krawiec, K.: Competent geometric semantic genetic programming for
  symbolic regression and boolean function synthesis.
\newblock Evolutionary Computation \textbf{26}(2), 177--212 (2018)

\bibitem{PoliLMBook}
Poli, R., Langdon, W.B., McPhee, N.F.: A Field Guide to Genetic Programming.
\newblock http://lulu.com (2008)

\bibitem{PoliVLM10}
Poli, R., Vanneschi, L., Langdon, W.B., McPhee, N.F.: Theoretical results in
  genetic programming: the next ten years?
\newblock Genetic Programming and Evolvable Machines \textbf{11}(3-4), 285--320
  (2010)

\bibitem{RoweS12}
Rowe, J.E., Sudholt, D.: The choice of the offspring population size in the (1,
  {\(\lambda\)}) {EA}.
\newblock In: Proceedings of the Genetic and Evolutionary Computation
  Conference ({GECCO} 2012), pp. 1349--1356 (2012)

\bibitem{ScharnowTW04}
Scharnow, J., Tinnefeld, K., Wegener, I.: The analysis of evolutionary
  algorithms on sorting and shortest paths problems.
\newblock Journal of Mathematical Modelling and Algorithms \textbf{3}(4),
  349--366 (2004)

\bibitem{Spector04:book}
Spector, L.: Automatic Quantum Computer Programming: A Genetic Programming
  Approach, \emph{Genetic Programming}, vol.~7.
\newblock Kluwer Academic Publishers, Boston/Dordrecht/New York/London (2004)

\bibitem{SpectorR02}
Spector, L., Robinson, A.J.: Genetic programming and autoconstructive evolution
  with the push programming language.
\newblock Genetic Programming and Evolvable Machines \textbf{3}(1), 7--40
  (2002)

\bibitem{UrliWN12}
Urli, T., Wagner, M., Neumann, F.: Experimental supplements to the
  computational complexity analysis of genetic programming for problems
  modelling isolated program semantics.
\newblock In: Proceedings of the 12th International Conference on Parallel
  Problem Solving from Nature ({PPSN} 2012), pp. 102--112 (2012)

\bibitem{Valiant84}
Valiant, L.G.: A theory of the learnable.
\newblock Communications of the {ACM} \textbf{27}(11), 1134--1142 (1984)

\bibitem{Valiant09}
Valiant, L.G.: Evolvability.
\newblock Journal of the {ACM} \textbf{56}(1), 3:1--3:21 (2009)

\bibitem{WagnerN12}
Wagner, M., Neumann, F.: Parsimony pressure versus multi-objective optimization
  for variable length representations.
\newblock In: Proceedings of the 12th International Conference on Parallel
  Problem Solving from Nature ({PPSN} 2012), pp. 133--142 (2012)

\bibitem{WagnerNU15}
Wagner, M., Neumann, F., Urli, T.: On the performance of different genetic
  programming approaches for the {SORTING} problem.
\newblock Evolutionary Computation \textbf{23}(4), 583--609 (2015)

\end{thebibliography}
\end{document}